\documentclass{article}
\usepackage{geometry,amsfonts} 
\usepackage{amsmath,amsfonts,amsthm,amssymb, graphicx}
\usepackage[caption=false]{subfig}
\usepackage{amsmath,amsfonts,amssymb, graphicx, setspace, verbatim}
\usepackage{multirow}
\usepackage{enumitem}
\usepackage{float}
\usepackage{todonotes}
\usepackage[ruled]{algorithm2e}
\usepackage{dsfont}
\usepackage{subfig}
\usepackage{amsmath}
\usepackage{amsfonts}
\usepackage{amssymb}
\usepackage{graphicx}
\usepackage{sidecap}
\usepackage{color}
\usepackage{hyperref}

\def\PP{{\mathbb P}}
\def\Prob{{\mathbb P}}
\def\E{{\mathbb E}}
\def\R{{\mathbb R}}
\def\Vhat{\widehat{V}}

\def\rbar{{\bar{r}}}

\def\fea{\varphi}

\def\pfea{\phi}
\def\var{\text{var}}

\def\rfea{\zeta}
\def\ghat{\hat{g}}
\def\phat{\hat{p}}
\def\calC{\mathcal{C}}
\def\calS{\mathcal{S}}
\def\bfP{\mathbf{P}}

\newtheorem{theorem}{Theorem}

\newtheorem{corollary}{Corollary}

\newtheorem{definition}{Definition}

\newtheorem{lemma}{Lemma}

\newtheorem{proposition}{Proposition}

\newtheorem{assumption}{Assumption}
\begin{document}
\title{Stochastic Gradient Descent with Dependent Data for Offline Reinforcement Learning}
\author{Jing Dong and Xin T. Tong}
\date{} 
\maketitle
\begin{abstract}
In reinforcement learning (RL), offline learning decoupled learning from data collection 
and is useful in dealing with exploration-exploitation tradeoff and enables data reuse in many applications. In this work, we study two offline learning tasks: policy evaluation and policy learning. 
For policy evaluation, we formulate it as a stochastic optimization problem and show that it can be solved using  approximate stochastic gradient descent (aSGD) with time-dependent data.
We show aSGD
achieves $\tilde O(1/t)$ convergence when the loss function is strongly convex and the rate is independent of the discount factor $\gamma$. 
This result can be extended to include algorithms making approximately contractive iterations such as TD(0).
The policy evaluation algorithm is then combined with the policy iteration algorithm to learn the optimal policy. To achieve an $\epsilon$ accuracy, the complexity of the algorithm is $\tilde O(\epsilon^{-2}(1-\gamma)^{-5})$, which matches the complexity bound for classic online RL algorithms such as Q-learning.
\end{abstract}

\section{Introduction}

Reinforcement learning (RL) is a fast developing area and has seen remarkable success in a wide range of applications.
There are various RL algorithms tailored to different application contexts. 
When designing RL algorithms, we need to distinguish whether the data is collected online or offline.
In the online setting, one can iteratively collect data by interacting
with the environment, and then use the data to improve the policy, which can be applied to collect more data in the follow-up iterations. 
However, in some important applications such as autonomous driving and healthcare, online interaction can be infeasible, expensive, or undesirable \cite{levine2020offline}.
In this context, it is more practical to use previously collected data, which is often referred to as the offline or complete off-policy setting. Offline RL is comparatively more challenging than online RL, and some of its theoretical properties are less understood than their online counterparts \cite{sutton2009fast,thomas2016data,yin2021near}. In this work, we focus on offline RL and show that under appropriate conditions, its complexity is similar to that of online RL.

We also distinguish between two RL tasks: policy evaluation and policy learning.
Policy evaluation refers to the task of estimating the value function for a given policy. While this task can be of independent interests in 
many applications, it also serves as an intermediate step in many policy learning algorithms such as least-squares policy iteration \cite{lagoudakis2003least} and proximal policy optimization \cite{schulman2017proximal}.
In this work, we study the policy evaluation problem in the offline setting where data is generated by the underlying Markov chain following some prefixed behavior policy. Utilizing a stochastic optimization framework, we develop a stochastic gradient based learning algorithm and establish the corresponding finite-sample performance bound. 
Our development allows linear parameterization of the value function via feature functions which can be applied to problems with infinite states. 
We further establish complexity bound for the policy learning task based on policy iteration with approximated policy evaluation.  

Lastly, we distinguish between two commonly used reward formulations in the non-episodic setting: long-run average reward and discounted reward. 
The long-run average reward is a more appropriate formulation in many applications \cite{marbach2000call,tangamchit2002necessity}. However, it has received much less attention than its discounted counterpart due to 
the lack  contraction properties \cite{Sutton:2018}. Our policy evaluation analysis provides a unified framework to both reward formulations, and
the convergence rate can be independent of the discount factor.   

\subsection{Related literature and our contribution}

\paragraph{Finite-time analysis}
Many RL algorithms, especially policy evaluation algorithms, can be viewed as gradient based updates \cite{sutton2009fast}.
In this context, tools from stochastic approximation has been applied to study the algorithm performance.
For example, ordinary differential equation based methods have been applied to establish convergence of different algorithms \cite{tsitsiklis1997analysis, borkar2000ode}.  
Recent works further develop finite-time performance bounds using drift based methods with properly constructed Lyapunov functions \cite{srikant2019finite, chen2021lyapunov}.
See also \cite{dalal2018finite,gupta2019finite,xu2019two} for finite-time analysis of two-time-scale algorithms.

Our first contribution is that we formulate the offline policy evaluation as a stochastic optimization problem and establish finite-time performance bound of a gradient descent based method
that is independent of the discount factor. There are two challenges when applying the gradient descent framework.
First, the data are not independent and identically distributed (i.i.d.). We overcome this challenge by utilizing suitable ergodicity
properties of the underlying Markov processes. In particular, we extend the convergence analysis of stochastic gradient descent (SGD) to accommodate time-dependent data.
Second, there are further gradient estimation errors beyond the stochastic noise. 
We explicitly characterize how the estimation error affects the convergence of the algorithm.

Several recent papers also study RL algorithms via the lens of SGD.
The paper \cite{liu2020finite} shows how the gradient-based temporal difference (TD) algorithm can be viewed
as a true gradient descent algorithm with a prima-dual saddle point objective function. Its performance
analysis focuses on i.i.d. observations. The work \cite{zhang2021average} extends the algorithm and analysis 
from discounted reward formulation to long-run average reward, but still assumes i.i.d. observations.
The paper \cite{wang2018finite} extends the discounted-reward analysis to time-dependent data.
Meanwhile, the work \cite{bhandari2018finite} establishes finite-time performance bounds for online TD algorithms for discounted-reward using SGD-based arguments. 
Compared with these works, 1) our analysis can deal with long-run average
reward formulation, 2) allows time-dependent data, 3) works in the offline setting, and 4) in the strongly convex or contraction case, achieves $\tilde O(1/T)$ convergence, which is better than the rate in \cite{wang2018finite, zhang2021average} and matches that in \cite{bhandari2018finite}.

\paragraph{SGD with time-dependent data}
As an important intermediate step, our analysis also provides new complexity bounds for SGD running on time-dependent data.
While this problem has been studied before by  \cite{agarwal2012generalization,sun2018markov}, our results are different in the following sense.
First, our results allow a more general notion of stochastic gradient or stochastic update direction, which is necessary for the analysis of the TD-SGD iteration and TD(0) iteration. 
Second, the earlier result assumes ``oracle" (exact) stochastic gradient. We extend the result to allow further approximate errors of the stochastic gradient.
Third, the bound in \cite{agarwal2012generalization} requires some prior knowledge of the ``regret", i.e., equation (6) in \cite{agarwal2012generalization}, which may not be readily available in practice. In comparison, we have more explicit conditions and step-size requirements.

\paragraph{Online optimization algorithm}
 Most existing algorithms that deal with long-run average reward in the offline setting operate on a fixed batch of data 
\cite{liu2018breaking, mousavi2020black, zhang2020gradientdice}. They estimate the ratio of the steady-state occupancy measure under the behavior policy and the target policy.
From the data processing perspective, our proposed optimization algorithm is online. In each iteration, we read one more data point and only update quantities related to the corresponding data point. The closest development to ours is \cite{wan2021learning}, but it only establishes asymptotic convergence of the algorithm, while we also provide non-asymptotic performance bounds.  Note that ``online/offline" can have different meanings in RL and optimization.
We study the convergence of an online optimization algorithm in offline RL. 

\paragraph{Policy learning} A direct application of our policy evaluation algorithm is to be combined with policy iteration to learn the optimal policy.
For the discounted reward formulation, we show that when the loss function under the linear parametrization is convex, to achieve an $\epsilon$-accuracy,
the total complexity is $\tilde O(\epsilon^{-4}(1-\gamma)^{-9})$. When the loss function is strongly convex or the stochastic update direction satisfies a contraction property on average, the total complexity is $\tilde O(\epsilon^{-5}(1-\gamma)^{-2})$. The later complexity bound matches that of online Q-learning algorithms \cite{qu2020finite}.

\subsection{Problem Setting}
We consider a Markov decision process (MDP), $\mathcal{M}=(\mathcal{S}, \mathcal{A}, r, P, \gamma)$
where $\mathcal{S}$ is the state space, $\mathcal{A}$ is the action space, $r(s,a)$ is the random reward
when the agent is in state $s\in \mathcal{S}$ and selects action $a\in \mathcal{A}$, $P(s^{\prime}|s,a)$ is the probability of transition to state $s^{\prime}$ in the next epoch given the current state $s$ and taken action $a$, $\gamma\in(0,1]$ is the discount factor. With a little abuse of notation, we also denote $P(r|s,a)$ as the probability of receiving reward $r$ given the current state $s$ and taken action $a$. We assume that given $(s,a)$, $s'$ and $r$ are independent.

A policy $\pi$ is a mapping from each state $s\in \mathcal{S}$ to a probability measure over the actions in $\mathcal{A}$.
Specifically, we write $\pi(a|s)$ as the probability of taking action $a$ when the agent is in state $s$. We denote $\E_s^{\pi}$ as the expectation of functionals of the underlying Markov process under policy $\pi$ starting from state $s$ and $\text{var}^\pi_s$ as the corresponding variance. 

For policy evaluation, given a target policy $\pi$, we want to compute the cumulative discounted reward starting from an initial state $s$, which is given by 
\begin{equation}
\label{eqn:Vgamma}
V_{\gamma}^{\pi}(s)=\E_s^{\pi} \left[\sum_{t=0}^{\infty}\gamma^tr(s_t,a_t)\right],
\end{equation}
where $\gamma<1$. For some applications, it is more reasonable to consider the long-run average reward and  the
average value bias function, which take the form
\begin{equation}
\label{eqn:rpi}
\bar r^{\pi} = \lim_{T\to \infty}\frac{1}{T}\E_s^{\pi}\left[\sum_{t=1}^T r(s_t,a_t)\right],
\quad \bar V^{\pi}(s)=\E_s^{\pi}\left[\sum_{t=0}^{\infty}r(s_t,a_t)-\bar r^{\pi}\right]. 
\end{equation}
Note that $\bar r^{\pi}$ does not depend on the initial state. 
%
It is well known that the long-run average reward can be viewed as a singular limit of the discounted reward as $\gamma\to 1$ \cite{blackwell1962discrete}, i.e., 
$\lim_{\gamma\rightarrow 1} \left(V_{\gamma}^{\pi}(s) - \bar r^{\pi}/(1-\gamma)\right)=\bar V^{\pi}(s)$.
Thus, we also refer to the long-run average case as the $\gamma=1$ case.

We consider the offline setting where we have access to a sequence of previously collected data: $(s_t, a_t, r_t, s_t^{\prime})_{0\leq t\leq T}$.
The data is generated by the underlying Markov chain under a behavior policy $b$. In particular, $a_t\sim b(a|s_t)$ and $s_{t+1}=s_t^{\prime}$. 
Note that $b$ can be different from the target policy $\pi$. 

We also introduce the following notations that are used throughout our subsequent development. For a given policy $\pi$, let $P^{\pi}$ denote the transition kernel of the Markov chain under policy $\pi$, i.e., $P^{\pi}(s,s')=\sum_{a}\pi(a|s)P(s'|s,a)$, $s,s'\in\mathcal{S}$. Under suitable ergodicity condition, this Markov chain has a stationary distribution, which we denote as $\mu^{\pi}$.
For a given measure $\mu$ on $\mathcal{S}$, we denote $\E_{\mu} f=\sum_{s\in \mathcal{S}} \mu(s)f(s)$ and $\var_{\mu}f=\sum \mu(s)(f(s)-\E_{\mu} f)^2$.
For two probability measures $\mu$ and $\nu$, we denote $\|\mu-\nu\|_{TV}$ as the total variation distance 
between $\mu$ and $\nu$. 

We use two notions to quantify the convergence rate of the Markov chain to stationarity: mixing time and spectral gap. Consider a generic Markov chain $s_t$ with transition kernel $P$ and stationary distribution $\mu$.
\begin{definition} 
The Markov chain has a mixing time $\tau<\infty$ if for any $s_0\in \mathcal{S}$ ,
\[\|P(s_{\tau}\in\,\cdot\,|s_0)-\mu(\cdot)\|_{TV}\leq 1/4.\] 
\end{definition}
The mixing time exists for irreducible and aperiodic Markov chains \cite{levin2017markov}.
\begin{definition}
\label{def:spectral}
The Markov chain has a spectral gap $\lambda$ if for any $\mu$-$l_2$ measurable function $f$, 
$\var_{\mu}(P f)\leq (1-\lambda)^2 \var_{\mu} f$.
\end{definition}
For a finite-state Markov chain, if the largest two eigenvalues of $P$ are $1$ and $\rho\in \mathbb{R}^+$, then 
it has a spectral gap $\lambda=1-\rho$. 

For a vector $\theta \in \mathbb{R}^d$, we denote $\|\theta\|$ as its Euclidean norm and
$\|\theta\|_{\infty}$ as its supremum norm. For a matrix $A$, we define $\|A\|=\sup_{x\neq 0}\|Ax\|/\|x\|$.
We also define $e_i$ as a unit vector whose $i$-th element is equal to $1$, i.e, the $i$-th basis vector, $I$ as the identity matrix, and ${\bf 1}$ and ${\bf 0}$
as vectors with all elements equal to $1$ and $0$ respectively.

Lastly, given two sequences of nonnegative real numbers $\{a_n\}_{n\ge 1}$ and $\{b_n\}_{n\ge 1}$, we define $b_n=O(a_n)$ and  if there exist some constants $C$ such that $ b_n \le C a_n$. We use $\tilde O$ when we ignore the logarithmic factors in $O$.

\section{Main Methodologies and Algorithms}
In this section, we introduce our main algorithm for policy evaluation.
The algorithm is based on applying SGD to a properly formulated
stochastic optimization problem.
\subsection{A unified loss function for policy evaluation}
Policy evaluation can be formulated as finding a solution to the Bellman equation.
Traditionally, the Bellman equation is discussed separately for the discounted $(\gamma<1)$ and long-run average $(\gamma=1)$ formulations. 
In particular, the Bellman operator for the discounted reward takes the form
\[
V_{\gamma}^{\pi}(s)=\E_s^{\pi}\left[r+\gamma V_{\gamma}^{\pi}(s^{\prime})\right]
=\sum_{a}\sum_{r}\pi(a|s)P(r|s,a)r + \gamma \sum_{a}\sum_{s^{\prime}}\pi(a|s)P(s^{\prime}|s,a)V_{\gamma}^{\pi}(s^{\prime})
.\]
When $\gamma=1$, the Bellman operator for the long-run average reward takes the form
\[
\bar V^{\pi}(s)=\E_s^{\pi}\left[r-\bar r + \bar V^{\pi}(s^{\prime})\right]
=\sum_{a}\sum_{r}\pi(a|s)P(r|s,a)r-\bar r+\sum_{a}\sum_{s^{\prime}}\pi(a|s)P(s^{\prime}|s,a)\bar V^{\pi}(s^{\prime}).
\]
We define the ``unified" Bellman operator as
\begin{equation}\label{eq:bellman}
V^{\pi}(s)=
\sum_{a}\sum_{r}\pi(a|s)P(r|s,a)r-\bar r+\gamma \sum_{a}\sum_{s^{\prime}}\pi(a|s)P(s^{\prime}|s,a)V^{\pi}(s^{\prime}),
\end{equation}
where we set $\bar r =0$ in the discounted version and $\gamma=1$ in the long-run average version. Such unification will make our discussion simpler.

Modern RL applications have large state spaces, so solving \eqref{eq:bellman} in practice often involves parameterizing the functional space of $V^\pi$ \cite{sutton1988learning,tsitsiklis1997analysis}.
One popular choice is to find a feature function $\fea:\mathcal{S}\to \R^d$, such that $V^\pi_\theta(s)=\fea(s)^T\theta$, where $\theta$ is a $d$-dimensional vector. 
Following this idea, we introduce another average-reward feature vector $\zeta$ when $\gamma=1$, so that 
$ \rbar_\theta=\rfea^T\theta.$ 
$\zeta$ is chosen such that $\zeta\bot \fea(s)$ for all $s$, since it is used to approximate $\bar{r}$. 
When $\gamma<1$, we can simply set $\zeta=0$. 
We assume there exists $\theta^*$ such that $V^\pi_{\theta^*}$ solves the Bellman equation  \eqref{eq:bellman}.

We also define 
\begin{equation}
\label{eq:predict}
\begin{split}
\xi^{\pi}(s)&=\sum_{a}\sum_{s^{\prime}} \sum_{r}\pi(a|s)P(r|s,a)r,
~~~ \phi^{\pi}(s)=\E^\pi_{s} \fea(s^{\prime})=\sum_{a}\sum_{s^{\prime}} \pi(a|s)P(s^{\prime}|s,a)\fea(s^{\prime})
\end{split}
\end{equation}
as the oracle conditional average reward and feature after transition respectively. Note that both $\xi^\pi$ and $\phi^\pi$ require knowledge of the transition probabilities. 
Thus, algorithms using them require certain estimations. This can introduce extra estimation errors and will be discussed in more details in Section \ref{sec:algorithm}. 
It is worth pointing out that if $\calS$ is finite,  we can choose $\fea$ and $\rfea$ as one-hot functions for each state, this parameterization is the same as the tabular formulation \cite{schaul2015universal}. We provide more details about the tabular parameterization and associated algorithms in the appendix.

Following the parametrization above, we consider the loss
\[
L(\theta,s)=\frac12(\E^\pi_{s}[r]-\bar{r}_\theta+\gamma \E^\pi_{s}[V_\theta^\pi(s')]-V_\theta^\pi(s))^2
=\frac{1}{2}(\xi^{\pi}(s)+(-\rfea+\gamma\phi^{\pi}(s)-\fea(s))^T\theta)^2.
\]
Recall that $b$ is the policy under which the data is collected, i.e., the behavior policy, and
$\pi$ is the policy we want to evaluate.
Let $\mu^b$  the invariant measure of the MDP under policy $b$ and define
\begin{equation}\label{eq:loss}
l(\theta):=\E_{\mu^b}L(\theta,s).
\end{equation}
As we will show next, solving the Bellman equation is equivalent to finding the minimizer of the stochastic optimization problem
$\min_{\theta} l(\theta)$. 

\begin{proposition}
\label{prop:optimality}
$V^\pi_{\theta^*}$ solves the Bellman equation \eqref{eq:bellman}  if and only if $\theta^*$ is a minimizer of $l(\theta)$.
\end{proposition}

\subsection{Policy evaluation with stochastic gradient descent} \label{sec:algorithm}
We consider applying gradient descent to solve $\min_{\theta}l(\theta)$. Note that 
\begin{equation}
\label{eq:gradl}
\nabla l(\theta)=\sum_s\mu^b(s) (\xi^\pi(s)+(-\rfea+\gamma \phi^\pi(s)-\fea(s)^T \theta)[(-\rfea+\gamma \pfea^\pi(s)-\fea(s)]
\end{equation}
There are two issues when evaluating $\nabla l(\theta)$ in practice. First, we do not have direct access to the oracle conditional average reward and feature after transition, i.e., $\xi^\pi$ and $\phi^\pi$ respectively. 
Instead, we approximate them using the corresponding sample averages. In particular, we estimate the transition probabilities via 
\begin{equation}
\label{eq:frequency}
\phat_t(s^{\prime}|s)=\sum_{a}\pi(a|s)\frac{\sum_{l\leq t} 1_{s_l=s, a_l=a, s_{l+1}=s'}}{\sum_{j\leq t} 1_{s_l=s, a_l=a}},
~~~ \phat_t(r|s)=\sum_{a}\pi(a|s)\frac{\sum_{l\leq t} 1_{s_l=s,a_l=a, r_l=r}}{\sum_{j\leq t} 1_{s_l=s,a_l=a}}.
\end{equation}
Then, the conditional average reward and feature after transition can be estimated via
$\hat\xi^{\pi}_t(s)=\sum_r r\phat_t(r|s)$ and 
$\hat\phi^{\pi}_t(s)=\sum_{s'}\phat_t(s'|s_t)\fea(s')$ respectively.

The second issue is the computational complexity of $\nabla l(\theta)$. When $|\mathcal{S}|$ is large, summing over all possible values of $s$ as in \eqref{eq:gradl} can be prohibitive. In this case, it is natural to consider SGD.  
Classic SGD requires having access to i.i.d. data \cite{kushner2003stochastic}. 
We start with the classic setting to derive the SGD updates and show later it works for time-dependent data under proper conditions as well.

Let $z=(s,a,r,s')$. We refer to a vector $g(\theta,z)$ as a stochastic gradient for $l$ if 
 \[
\bar g(\theta):=\sum_{s,a,r,s'}\mu^b(s)b(a|s)p(s'|a,s)p(r|a,s) g(\theta, z)=\nabla l(\theta).
\]
For notational simplicity, we let 
$\delta(\theta, s)=(\xi^\pi(s)+(-\rfea+\gamma \phi^\pi(s)-\fea(s))^T \theta)$. 
Next we discuss two specific forms of stochastic gradient. 
The first one is the gradient of $L(\theta,s)$ and takes the form 
 \begin{equation}
 \label{eq:direct}
 g(\theta,z)= g(\theta,s)=\delta(\theta, s)(-\rfea+\gamma \phi^\pi(s)-\fea(s)).
 \end{equation}
The second one is a generalization of GTD2 \cite{sutton2009fast} 
and takes the form 
\begin{equation}
\label{eq:TD}
g(\theta,z)=\frac{\pi(a|s)}{b(a|s)}\delta(\theta, s)(-\rfea+\gamma \fea(s')-\fea(s)).
\end{equation} 
The validity of \eqref{eq:TD} is verified in the following Lemma.
\begin{lemma}
\label{lem:checksutton}
The gradient in \eqref{eq:TD} is a stochastic gradient of $l(\theta)$
\end{lemma}
Compared to online RL, offline RL needs to deal with the \emph{distribution shift} since $\pi$ is different from $b$ \cite{levine2020offline}.
For \eqref{eq:TD}, we impose the additional weight $\pi(a|s)/b(a|s)$ to address this issue. 
Note that $\pi$ cannot be too far from $b$, e.g., $\pi$ cannot take actions that are never taken by $b$.

Using either one of the two stochastic gradient formulations, we arrive at an SGD algorithm:
$\theta_{t+1}=\theta_t-\eta_t\ghat_t(\theta_t,z_t)$, 
where $\eta_t$ is the step size and $\ghat_t$ is the sample average approximation of $g$ where we replace $\xi^{\pi}$ and $\phi^{\pi}$ with $\hat\xi_t^\pi$ and $\hat\phi_t^\pi$ respectively.
In what follows, we refer to SGD updates based on \eqref{eq:direct} as direct-SGD and SGD updates based on \eqref{eq:TD} as TD-SGD. The details of the algorithm is summarized in Algorithm \ref{alg:main}. For the simplicity of analysis, we add  an extra projection step  at each iteration, i.e., we use projected SGD \cite{shalev2014understanding}. 
Mathematically, this operation is denoted by  a mapping $\bfP_{\mathcal{C}}$ where $\mathcal{C}$ is a properly defined convex set and
\[\bfP_{\mathcal{C}}(\theta)=\arg\min_{\theta' \in \mathcal{C}}\|\theta-\theta'\|, \] 
i.e., the projection onto the set $\mathcal{C}$.
In practice, $\mathcal{C}$ is usually picked as a large box region where the true parameter has to be inside. 

Note that when $s_t$'s are i.i.d.\ draws from $\mu^b$, and $a_t, s_{t}'$ are generated using policy $b$, $z_t$'s are i.i.d.. However, in practice, it is more often the case that $z_t=(s_t,a_t, r_t, s_{t}')$ are generated from the underlying Markov chain with $s_{t+1}=s_t'$. In this case, $z_t$'s are time-dependent.

From the data processing perspective, to update Algorithm \ref{alg:main} based on direct-SGD or TD-STD, one does not need access to $\{z_{s}, s\neq t\}$. 
For the two stochastic gradient formulations \eqref{eq:direct} and \eqref{eq:TD}, the theoretical analysis for \eqref{eq:TD} is more evolved. In addition, implementing \eqref{eq:TD} requires knowledge of $b(a|s)$, which may not be readily available in some applications. On the other hand, updates according to \eqref{eq:TD} resemble the standard TD algorithm and are sparse if $\fea$ and $\zeta$ are sparse. 

\subsection{Strong convexity, contraction, and TD(0)}
The linear parameterization ensures that $l(\theta)$ is a convex function. 
With certain choices of the feature vectors $\fea$ and $\zeta$, $l(\theta)$ can be  strongly convex. 
Strong convexity usually leads to better convergence rate.  
In particular, one theoretical advantage of strong convexity is that the gradient descent direction is also a  contraction direction towards $\theta^*$ on average.
\begin{definition}\label{def:contraction}
A vector field $\bar g(\theta)$ generates a $c$-contraction towards $\theta^*$ if 
$\langle \theta-\theta^*,  g(\theta)\rangle\geq c \|\theta-\theta^*\|^2$ and $\bar g(\theta^*)=0.$
\end{definition}

Other algorithms can also iterates along stochastic directions which is contractive on average. A popular one as we will explain next is temporal difference learning (TD(0)) \cite{sutton1988learning}. 
In offline policy evaluation, we still consider iterations of the form $\theta_{t+1}=\theta_t-\eta_tg(\theta_t,z_t)$. However, $g$ now takes the form
\begin{equation}
\label{eq:TD0}
g(\theta, z_t)=-\frac{1}{\mu^b(s_t)}\frac{\pi(a_t|s_t)}{b(a_t|s_t)}(r_t-\zeta^T\theta+\gamma \fea(s_{t+1})^T \theta-\fea(s)^T\theta)(\fea(s_t)+\zeta)
\end{equation}
Comparing to online TD(0), \eqref{eq:TD0} requires additional weight adjustment, i.e., $\frac{1}{\mu^b(s_t)}\frac{\pi(a_t|s_t)}{b(a_t|s_t)}$, to account for the distribution shift.
 The actual algorithm is given in Algorithm \ref{alg:main} as well. Since TD(0) does not require knowledge of $\xi^\pi$ and $\phi^\pi$, it can be easier to implement than direct-SGD or TD-SGD. However, it does require knowledge of the invariant measure $\mu^b$. If $\mu^b$ is not known, we can estimate it using the sample average, i.e., $\hat\mu_t^b(s)=\frac{1}{t}\sum_{l=1}^{t}1_{s_l=s}$. We denote $\hat g_t(\theta, z_t)$ as an approximation of $g(\theta, z_t)$ by replacing $\mu^b(s)$ with $\hat\mu_t^b(s)$. 

\begin{algorithm}[htp!]
 \KwIn{Data sequence $z_t=(s_t,a_t,r_t,s_t^{\prime})$, sampling policy $b$, evaluation policy $\pi$, feature functions $\fea,\rfea$,  discount factor $\gamma\in(0,1]$, feasibility set $\calC$, initial parameter $\theta_0\in\calC$}
 \KwOut{Value function $V^\pi(s)=\fea(s)^T\bar{\theta}_T$ ( and  average reward $\rbar=\rfea^T \bar{\theta}_T$ )}
 \For{$t=1$ \KwTo $T$}{
Update $\hat g_t(\theta_t, z_t)$
according to direct-SGD \eqref{eq:direct}, TD-SGD \eqref{eq:TD}, or TD(0) \eqref{eq:TD0} using sample average approximation.

Set $\theta_{t+1}=\theta_t - \eta_t\hat g_t(\theta_t,z_t)$.

Apply $\bfP_\calC$ to $\theta_{t+1}$.
 }
Return $\bar{\theta}_T=\frac{1}{T+1}\sum_{t=0}^T \theta_t$.  
\caption{Policy evaluation with feature functions}
\label{alg:main}
\end{algorithm}

In TD(0), we note that
\begin{equation}
\label{eq:avgTD0}
\begin{split}
\bar{g}(\theta):&=\sum_{s,a,r,s'} \mu^b(s)b(s|a)p(s'|s,a)p(r|s,a) g(\theta, z)\\
&=-\sum_s (\xi^\pi(s)+(-\zeta+\gamma \phi^\pi(s)-\fea(s))^T\theta)(\fea(s)+\zeta).
\end{split}
\end{equation}
It is worth mentioning that $\bar g$ is not the gradient of $l$ anymore. Thus, we only refer to direct-SGD and TD-SGD as SGD-based updates.
We next show that both SGD and TD(0) can induce a contraction on average in its updating direction.
Let
\[
D=\sum_s \fea(s)\fea(s)^T-\gamma  \fea (s)\phi^\pi(s)^T + \rfea \rfea^T.
\] 
We define $A\succeq B$ if $A-B$ is positive semi-definite.
\begin{lemma}
\label{lem:gen}
In direct-SGD and TD-SGD, $\bar g$ is a c-contraction if $l$ is $c$-strongly convex, i.e.,
$\nabla^2 l \succeq cI$.
In $TD(0)$, $\bar g$ is a c-contraction if 
$\frac{1}{2} (D+D^T)\succeq cI$.
\end{lemma}

From Lemma \ref{lem:gen}, the contraction property depends on the choice of features. To gain insights into this, we investigate a simplified case where the state space if finite and $\calS=\{s_1,\ldots, s_{|\calS|}\}$. 
When $\gamma<1$, the conditions in Lemma \ref{lem:gen} are met with standard tabular parameterization: 

\begin{proposition}\label{lm:strongconvex}
 When $\gamma<1$, under the tabular parameterization, i.e., $\fea(s_i)=e_i$, the conditions in Lemma \ref{lem:gen} hold with $\nabla^2 l\succeq \min_s \mu^b(s)(1-\gamma)^2 I$ and $\frac12(D+D^T)\succeq (1-\gamma) I$.
\end{proposition} 

When $\gamma=1$, the standard tabular parameterization no longer satisfies the conditions in Lemma \ref{lem:gen}.
In particular, when $\gamma=1$, $V^{\pi}(s)$ is not uniquely identifiable. We can fix the issue by
imposing constraints on the basis as demonstrated in the following proposition.


\begin{proposition}\label{lm:strongconvex1}
When $\gamma=1$, we set $\zeta=e_{|\calS|}$ and consider the following feature specifications:
\begin{enumerate}
\item Suppose $\mu^\pi(s_{|\calS|})>0$. Set $\fea(s_i)=e_i, i=1,\ldots, |\calS|-1$, and $\fea(s_{|\calS|})={\bf 0}$. Then, the conditions in Lemma \ref{lem:gen} hold with 
$\nabla^2l\succeq \min_s\mu^b(s)(1-p)^2I$ and $\frac12(D+D^T)\succeq (1-p)I$ where $p=\|\tilde P^{\pi}\|$ and $\tilde{P}^\pi$  is an ${|\calS|\times |\calS|}$ matrix with the first $|\calS|-1$ columns the same as the first $|\calS|-1$ columns of $P^\pi$, while the last column is zero.
\item Suppose $P^\pi$ has a spectral gap $\lambda$. Let $U\in R^{|\calS|\times (|\calS|-1)}$ with orthonormal column vectors that are orthogonal to $\vec{1}$. Set $\fea(s_i)^T=[u_i,0]$ where $u_i$ is the $i$-th row of $U$. Then, the conditions in Lemma \ref{lem:gen} hold with $\nabla^2l\succeq \min_s \mu^b(s)\lambda^2 I$ and $\frac12(D+D^T)\succeq \lambda I$.
\end{enumerate}
\end{proposition} 


\section{Finte-time analysis for Algorithm \ref{alg:main}} \label{sec:comp_evaluation}
In this section, we study the convergence rate for the proposed policy evaluation algorithms.
We also conduct convergence analysis for general aSGD and approximate stochastic contraction with time-dependent data. 

\subsection{Approximate stochastic gradient or contraction for dependent data} \label{sec:SGD}
We first consider a generic approximate stochastic iterative algorithm, which can be taken out of the context of policy evaluation. For a given convex feasibility set $\mathcal{C}$, we consider the iterates
\begin{equation}
\label{eq:genSGD}
\theta_{t+1}=\bfP_{\mathcal{C}}(\theta_t-\eta_t \ghat_t(\theta,z_t)),
\end{equation}
where $z_t$, $t=0,1,\dots$, is the data sequence and $\ghat_t$ is an approximation of a properly defined stochastic gradient or contraction (see Assumption \ref{aspt:estimate}).
To study the convergence of \eqref{eq:genSGD}, we need to impose some regularity conditions on the data generating process and our estimation procedure.

First, we assume the data sequence comes from an ergodic Markov process:
\begin{assumption}
\label{aspt:ergodic}
$z_t$ is an ergodic Markov chain on some Polish space $\mathcal{Z}$ with a stationary distribution $\mu$. 
There exists a finite mixing time $\tau<\infty$. 
\end{assumption}

Second, we define the notion of an oracle stochastic gradient or contraction. 
In practice, we often needs further approximation of the oracle, i.e.,
$\ghat_t$ in \eqref{eq:genSGD}. We refer to \eqref{eq:genSGD} as an approximate stochastic gradient or contraction in the following sense:
\begin{assumption}
\label{aspt:estimate}
There is an accurate stochastic movement $g(\theta,z_t)$ with 
$\bar g(\theta):=\sum_{z}\mu(z) g(\theta, z)$.
There exists a stochastic sequence $e_t$ with $\E e_t^2<\infty$ such that
$\|g(\theta,z_t)-\ghat_t(\theta,z_t)\|\leq e_t$.
\end{assumption}
If $\bar g(\theta)$ is a gradient of $l(\theta)$, we refer to $g(\theta,z_t)$ as an oracle stochastic gradient and iterates according to \eqref{eq:genSGD} as aSGD.
If $\bar g(\theta)$ is a $c$-contraction according to Definition \ref{def:contraction}, we refer to $g(\theta,z_t)$ as an oracle stochastic contraction.
Note that when $l(\theta)$ is strongly convex and $\bar g(\theta)=\nabla l(\theta)$, $g(\theta,z_t)$ is both a stochastic gradient and a stochastic contraction.

Lastly, we assume the projection set is compact and the associated gradients are bounded and smooth.
\begin{assumption}
\label{aspt:bound}
$\theta^*\in \mathcal{C}$.
There exist constants $C_0,C_1>0$, such that 
\[
\|\bar g\|_{\infty}, \|\ghat_t\|_{\infty},\|g_t\|_{\infty} \leq C_0,\quad \calC\subset \{\theta:\|\theta\|\leq C_1\}.
\]
There exists $G>0$, such that $g(\theta, z)$ is $G$-Lipschitz continuous in $\theta$ on $\mathcal{C}$ for any $z\in\mathcal{Z}$.
\end{assumption}
The requirement that $\calC$ is in a ball of radius $C_1$ is achievable for a properly chosen $C_1<\infty$. Note that $C_1$ in general needs to increase with the dimension of $\theta$, which may further depend on the dimension of $z_t$. 


\begin{theorem}
\label{thm:main}
Under Assumptions \ref{aspt:ergodic} -- \ref{aspt:bound}, if $\bar g(\theta)=\nabla l(\theta)$, the average of iterates $\bar{\theta}_T=\frac1T\sum \theta_t$ from  iteration \eqref{eq:genSGD} satisfies
\[
\E l(\bar{\theta}_T)-l(\theta^*)=O\left(\frac{1}{T}\left(\tau \log T \sum_{t=0}^{T}\eta_t+\frac{1}{\eta_T}+\sum_{t=0}^T\E e_t \right)\right).
\]
where $O$ hides a polynomial of $C_0$, $C_1$, and $G$. 
If we fix $\eta_t=\eta_0 t^{-1/2}$ for some $\eta_0>0$, and assume $\tau=O(1)$ and 
$\E e_t =O(1/\sqrt{t})$, then we can further simplify the bound to 
$\E l(\bar{\theta}_T)-l(\theta^*)=O\left(\log T/\sqrt{T}\right)$.
\end{theorem}


When the average movement $\bar{g}$ is a $c$-contraction, we can further improve the convergence rate in Theorem \ref{thm:main} from $\tilde O(1/\sqrt{T})$ to $\tilde O(1/T)$. 
This require us to set  the step size as follows: $\eta_1\in[1/c, 2/c]$ and for $t\geq 1$,
\begin{equation}\label{eq:step}
\eta_{t+1}=\exp\left(-\frac{1}{2}c\eta_t\right)\eta_t.
\end{equation}
Note that in practice we may not know the exact value of $c$ but rather a lower bound $c_0\leq c$. Then, it suffices to set the step sizes as in $\eqref{eq:step}$ with $c$ replaced by $c_0$. 

\begin{theorem} \label{th:main2}
Under Assumptions \ref{aspt:ergodic} -- \ref{aspt:bound},  if $\bar{g}$ is a c-contraction, and with the choice of step sizes defined in \eqref{eq:step}, 
the iterates according to \eqref{eq:genSGD} satisfy
\[
\E\|\theta_{T}-\theta^*\|^2=O\left(\frac{\tau^2(\log T)^3}{T} + \frac{1}{T}\sum_{t=1}^{T}\E e_t^2\right),
\]
where $O$ hides a polynomial of $C_0$, $C_1$, and $G$. 
If we have $\tau=O(1)$ and $\E e_t^2=O(\frac 1t)$, we can further simplify the bound to
$\E\|\theta_{T}-\theta^*\|^2=O\left((\log T)^3/T\right)$.
\end{theorem}

We make several observations from Theorems \ref{thm:main} and \ref{th:main2}.
First, when $\bar{g}=\nabla l$, 
aSGD has the same convergence rate as the oracle SGD as long as the estimation error $\E e_t^2=O(1/t)$.
Note that $\E e_t^2=O(1/t)$ is the standard error bound for sampling-based estimators, including the sample-average estimator in \eqref{eq:frequency}. 
Second, using Markovian data sequence has the same convergence rate as i.i.d. data as long as the underlying Markov chain is ergodic with mixing time $\tau=O(1)$. In particular, $\tilde O(1/\sqrt{T})$ and $\tilde O(1/T)$ are the convergence rate for standard SGD for convex and strongly convex loss functions respectively \cite{shalev2014understanding}.

\subsection{Application to policy evaluation}
In this subsection, we demonstrate how the results developed in Section \ref{sec:SGD} can be applied to policy evaluation, i.e., Algorithm \ref{alg:main}.

We first note that in practice, it is reasonable to assume that the rewards are properly bounded. Since $\fea(s)$ and $\zeta$ are basis vectors, they are usually bounded as well. 

\begin{assumption}
\label{aspt:bound2}
$\theta^*\in\mathcal{C}$.
There exist constants $\tilde C_0,\tilde C_1>0$, such that $|r|<\tilde M$ with probability 1,
$\|\rfea\|_{\infty}, \|\fea\|_{\infty} \leq \tilde  C_0$, and $\calC\subset \{\theta:\|\theta\|\leq \tilde C_1\}$.
\end{assumption}

We also impose an ergodicity assumption on the data collecting policy $b$. These assumptions are typically assumed in the literature (see, for example, \cite{sutton2009fast, bhandari2018finite}).

\begin{assumption}
\label{aspt:ergodic2}
The data collecting policy $b$ generates an irreducible ergodic chain on $\mathcal{S}$. 
There exists a mixing time $\tau<\infty$. 
\end{assumption}

Lastly, we assume the estimators for the transition probability and invariance measure are reasonably accurate, which can be achieved by standard sample-average estimators. 

\begin{assumption}
\label{aspt:estimate2}
There exists a stochastic sequence $\tilde e_t$ such that the estimator of the transition probability $\phat_t$ and invariance probability   $\hat\mu^b$ satisfy 
$\|\phat_t(\cdot|s_t)-p(\cdot|s_t)\|_{\infty}\leq \tilde e_t$, $|1/\hat\mu^b_t(s_t)-1/\mu^b(s_t)|\leq \tilde e_t$, and $\E \tilde e_t^2=O(1/t)$.
\end{assumption}

We can verify that under Assumptions \ref{aspt:bound2} -- \ref{aspt:estimate2}, Assumptions \ref{aspt:ergodic} -- \ref{aspt:bound}
hold for Algorithm \ref{alg:main}. This gives the following finite-time performance bound.
\begin{theorem} \label{thm:main3}
Suppose $\pi(a|s)\leq C\mu^b(s)b(a|s)$ for some $C>0$. Under Assumptions \ref{aspt:bound2} -- \ref{aspt:estimate2}, the SGD-based updates in Algorithm \ref{alg:main} solves the offline policy evaluation problem with 
\[
\E l(\bar{\theta}_T)-l(\theta^*)=O\left(\log T/\sqrt{T}\right).
\]
If $l(\theta)$ is strongly convex,
with the choice of step sizes defined in \eqref{eq:step}, we further have 
\begin{equation}
\label{eqn:fastrate}
\E \|\bar{\theta}_T-\theta^*\|^2=O\left((\log T)^3/T\right), \quad
\E l(\bar{\theta}_T)-l(\theta^*)=O\left((\log T)^3/T\right). 
\end{equation}
If $\bar{g}$ in TD(0) is a contraction, the TD(0)-based update in Algorithm \ref{alg:main} solves the offline policy evaluation problem with rate \eqref{eqn:fastrate}. 

\end{theorem}

It is worth pointing out that having a  low $l(\bar{\theta}_T)$ value only indicates that $V_{\bar{\theta}_T}$ fits the Bellman equation \eqref{eq:bellman} well. 
However, it is a priori unclear whether $V_{\bar{\theta}_T}$ is close to the exact value function $V^\pi=:V_{\theta^*}$.
The next result addresses this problem assuming the Markov process under policy $\pi$ has a spectral gap and the distribution shift from $\mu^\pi$ to $\mu^b$ is not large.

\begin{theorem} \label{thm:main1}
Suppose $P^\pi$ has a spectral gap $\lambda>0$ and $\mu^\pi(s) \leq C\mu^b(s)$ for some $C\in(0,\infty)$.  
\begin{enumerate}
\item If $\gamma<1$, we have 
\[
\sum  \mu^\pi(s) (V^{\pi}(s)-V_{\theta}(s))^2\leq \frac{C}{(1-\gamma(1-\lambda))^2}(l(\theta)-l(\theta^*)). 
\]
\item If $\gamma=1$, we have 
\[
\min_{a}\sum  \mu^\pi(s) (V^{\pi}(s)+a-V_{\theta}(s))^2+(\bar r^{\pi}-\bar r_\theta)^2\leq \frac{C}{\lambda^2}(l(\theta)-l(\theta^*)). 
\]
\end{enumerate}

\end{theorem}

The challenge imposed by the distribution shift can be seen by the requirement that $\mu^\pi(s) \leq C\mu^b(s)$. 
In particular, we note from Theorem \ref{thm:main1} that the loss function $l(\theta)$ is only a good indicator of the accuracy of  $V^\pi$ if $C$ is of a reasonable size.

\section{Policy iteration with approximated policy evaluation} \label{sec:comp_iteration}
Policy iteration is  a standard way to handle the policy learning task. It can also be implemented in the offline setting. Policy evaluation is an important subroutine in policy iteration. Given our results in Theorem \ref{thm:main3}, we next discuss how our estimates can be integrated with policy iteration. Since the convergence rate of policy iteration is only studied for the discounted case, i.e., $\gamma<1$,  (see \cite{bertsekas1995dynamic,bertsekas2011approximate}), we restrict our discussion to this setting as well.

Policy iteration iteratively updates the policy according to
\[
\tilde{\pi}_{k+1}=\arg\max_{\pi\in \Pi} \E^{\pi} r(s,a) +\gamma \E^{\pi} \sum_{s'} V^{\tilde{\pi}_k}(s') P^{\pi}(s,s'),
\]
where $\Pi$ denote the space of admissible policies. We assume the optimal policy $\pi^*\in \Pi$. 
In what follows, we denote quantities related to the optimal policy with a superscript $*$. For example, the optimal value function is $V^*$, the 
stationary distribution of the underlying Markov chain under the optimal policy is $\mu^*$.

When $V^{\pi_k}$ can be evaluated exactly and $\gamma<1$, it is well known that policy iteration has linear convergence. 
However, in practice, we only have access to an estimate of $V^{\pi_k}$, which we denote as $\Vhat^{\pi_k}$.
Define the one step policy iteration with approximated policy evaluation as
\begin{equation}
\label{eqn:policyiter}
\pi_{k+1}=\arg\max_{\pi\in \Pi} \E^{\pi} r(s,a) +\gamma \E^{\pi} \sum_{s'} \Vhat^{\pi_k}(s') P^{\pi}(s,s').
\end{equation}
We next develop a general error bound for iteration based on \eqref{eqn:policyiter}. 
Classic error bounds for policy iteration with approximated policy evaluation is based on the $l_\infty$ norm \cite{bertsekas2011approximate}, which is not compatible with Theorem \ref{thm:main3}. We use the $\mu^*$-$l_1$ norm instead. 

\begin{theorem} \label{th:approx_iteration}
Suppose the policy evaluation error at the $k$-th iteration satisfies
\begin{equation}
\label{eqn:cond1}
\E \sum_s \mu^{\pi_k}(s)|V^{\pi_k}(s)-\Vhat^{\pi_k}(s)|\leq \epsilon_k,
\end{equation}
where the expectation is taken over the randomness in policy evaluation. 
In addition, suppose $\mu^*(s)\leq C\mu^{\pi_{k+1}}(s)\leq C^2\mu^{\pi_k}(s)$ for some $C\in(0,\infty)$. 
Then,
\[
\E \sum_{s} \mu^*(s)(V^{*}(s)-V^{\pi_{K}}(s))\leq \gamma^K\E\sum_s\mu^*(s)(V^{\pi^*}(s)-V^{\pi_{0}}(s))+\frac{\sum_{k=1}^K\gamma^{K-k}C^2\epsilon_k}{1-\gamma}.
\]
\end{theorem}

Based on Theorem \ref{th:approx_iteration}, we can quantify the complexity of policy iteration with approximated policy evaluation via Algorithm \ref{alg:main}.
Given any accuracy level $\epsilon>0$, we consider the error requirement $\E \sum_{s} \mu^*(s)(V^{*}(s)-V^{\pi_{K}}(s))=\epsilon$.

\begin{corollary} \label{cor:complex}
Consider policy iteration \eqref{eqn:policyiter} with approximate policy evaluation subroutine Algorithm \ref{alg:main}.
Suppose Assumptions \ref{aspt:bound2} -- \ref{aspt:estimate2} holds. In addition, assume $\pi(a|s)\leq C\mu^b(s)b(a|s)$ and $\mu^{\pi}(s)\leq C\mu^b(s)$ for some $C>0$. Then, for any given $\epsilon>0$,
SGD-based updates achieve an $\epsilon$-accuracy 
with an overall sampling complexity
\[
O \left(\frac{\log(1/\epsilon)^2\log(1/(1-\gamma))}{(1-\gamma)^9\epsilon^4}\right).
\]
If $l$ is strongly convex in SGD-based updates or $\bar g$ is a c-contraction in TD(0), we can achieve an $\epsilon$-accuracy with an overall sampling complexity
\[
O\left(\frac{\log(1/\epsilon)^2\log(1/(1-\gamma))}{(1-\gamma)^5\epsilon^2}\right).
\]
\end{corollary}

\section{Conclusion}
In this paper, we demonstrate how to apply stochastic optimization techniques to develop and analyze
offline RL algorithms. We derive finite-time performance bounds for SGD-based updates and TD(0) for offline policy evaluation with time-dependent data.
These bounds are independent of the discount factor and can be applied to both discounted reward and long-run average reward formulations.
Our development also extends convergence results for SGD and stochastic contraction with time-dependent data.


\bibliographystyle{plain}
\bibliography{RL_ref}

\begin{thebibliography}{10}

\bibitem{agarwal2012generalization}
Alekh Agarwal and John~C Duchi.
\newblock The generalization ability of online algorithms for dependent data.
\newblock {\em IEEE Transactions on Information Theory}, 59(1):573--587, 2012.

\bibitem{bertsekas2011approximate}
Dimitri~P Bertsekas.
\newblock Approximate policy iteration: A survey and some new methods.
\newblock {\em Journal of Control Theory and Applications}, 9(3):310--335,
  2011.

\bibitem{bertsekas1995dynamic}
D.P. Bertsekas.
\newblock {\em Dynamic Programming and Optimal Control}, volume~1.
\newblock Athena scientific Belmont, MA, 1995.

\bibitem{bhandari2018finite}
Jalaj Bhandari, Daniel Russo, and Raghav Singal.
\newblock A finite time analysis of temporal difference learning with linear
  function approximation.
\newblock In {\em Conference on learning theory}, pages 1691--1692. PMLR, 2018.

\bibitem{blackwell1962discrete}
David Blackwell.
\newblock Discrete dynamic programming.
\newblock {\em The Annals of Mathematical Statistics}, pages 719--726, 1962.

\bibitem{borkar2000ode}
Vivek~S Borkar and Sean~P Meyn.
\newblock The ode method for convergence of stochastic approximation and
  reinforcement learning.
\newblock {\em SIAM Journal on Control and Optimization}, 38(2):447--469, 2000.

\bibitem{chen2021lyapunov}
Zaiwei Chen, Siva~Theja Maguluri, Sanjay Shakkottai, and Karthikeyan Shanmugam.
\newblock A lyapunov theory for finite-sample guarantees of asynchronous
  q-learning and td-learning variants.
\newblock {\em arXiv preprint arXiv:2102.01567}, 2021.

\bibitem{dalal2018finite}
Gal Dalal, Gugan Thoppe, Bal{\'a}zs Sz{\"o}r{\'e}nyi, and Shie Mannor.
\newblock Finite sample analysis of two-timescale stochastic approximation with
  applications to reinforcement learning.
\newblock In {\em Conference On Learning Theory}, pages 1199--1233. PMLR, 2018.

\bibitem{gupta2019finite}
Harsh Gupta, R~Srikant, and Lei Ying.
\newblock Finite-time performance bounds and adaptive learning rate selection
  for two time-scale reinforcement learning.
\newblock {\em Advances in Neural Information Processing Systems},
  32:4704--4713, 2019.

\bibitem{kushner2003stochastic}
Harold Kushner and G~George Yin.
\newblock {\em Stochastic approximation and recursive algorithms and
  applications}, volume~35.
\newblock Springer Science \& Business Media, 2003.

\bibitem{lagoudakis2003least}
Michail~G Lagoudakis and Ronald Parr.
\newblock Least-squares policy iteration.
\newblock {\em Journal of machine learning research}, 4(Dec):1107--1149, 2003.

\bibitem{levin2017markov}
David~A Levin and Yuval Peres.
\newblock {\em Markov chains and mixing times}, volume 107.
\newblock American Mathematical Soc., 2017.

\bibitem{levine2020offline}
Sergey Levine, Aviral Kumar, George Tucker, and Justin Fu.
\newblock Offline reinforcement learning: Tutorial, review, and perspectives on
  open problems.
\newblock {\em arXiv preprint arXiv:2005.01643}, 2020.

\bibitem{liu2020finite}
Bo~Liu, Ji~Liu, Mohammad Ghavamzadeh, Sridhar Mahadevan, and Marek Petrik.
\newblock Finite-sample analysis of proximal gradient td algorithms.
\newblock {\em arXiv preprint arXiv:2006.14364}, 2020.

\bibitem{liu2018breaking}
Qiang Liu, Lihong Li, Ziyang Tang, and Dengyong Zhou.
\newblock Breaking the curse of horizon: Infinite-horizon off-policy
  estimation.
\newblock {\em arXiv preprint arXiv:1810.12429}, 2018.

\bibitem{marbach2000call}
Peter Marbach, Oliver Mihatsch, and John~N Tsitsiklis.
\newblock Call admission control and routing in integrated services networks
  using neuro-dynamic programming.
\newblock {\em IEEE Journal on selected areas in communications},
  18(2):197--208, 2000.

\bibitem{mousavi2020black}
Ali Mousavi, Lihong Li, Qiang Liu, and Denny Zhou.
\newblock Black-box off-policy estimation for infinite-horizon reinforcement
  learning.
\newblock {\em arXiv preprint arXiv:2003.11126}, 2020.

\bibitem{qu2020finite}
Guannan Qu and Adam Wierman.
\newblock Finite-time analysis of asynchronous stochastic approximation and $ q
  $-learning.
\newblock In {\em Conference on Learning Theory}, pages 3185--3205. PMLR, 2020.

\bibitem{schaul2015universal}
Tom Schaul, Daniel Horgan, Karol Gregor, and David Silver.
\newblock Universal value function approximators.
\newblock In {\em International conference on machine learning}, pages
  1312--1320. PMLR, 2015.

\bibitem{schulman2017proximal}
John Schulman, Filip Wolski, Prafulla Dhariwal, Alec Radford, and Oleg Klimov.
\newblock Proximal policy optimization algorithms.
\newblock {\em arXiv preprint arXiv:1707.06347}, 2017.

\bibitem{shalev2014understanding}
Shai Shalev-Shwartz and Shai Ben-David.
\newblock {\em Understanding machine learning: From theory to algorithms}.
\newblock Cambridge university press, 2014.

\bibitem{srikant2019finite}
Rayadurgam Srikant and Lei Ying.
\newblock Finite-time error bounds for linear stochastic approximation andtd
  learning.
\newblock In {\em Conference on Learning Theory}, pages 2803--2830. PMLR, 2019.

\bibitem{sun2018markov}
Tao Sun, Yuejiao Sun, and Wotao Yin.
\newblock On markov chain gradient descent.
\newblock {\em arXiv preprint arXiv:1809.04216}, 2018.

\bibitem{Sutton:2018}
R.~Sutton and A.~Barto.
\newblock {\em Reinforcement Learning: An Introduction}.
\newblock MIT Press, 2 edition, 2018.

\bibitem{sutton1988learning}
Richard~S Sutton.
\newblock Learning to predict by the methods of temporal differences.
\newblock {\em Machine learning}, 3(1):9--44, 1988.

\bibitem{sutton2009fast}
Richard~S Sutton, Hamid~Reza Maei, Doina Precup, Shalabh Bhatnagar, David
  Silver, Csaba Szepesv{\'a}ri, and Eric Wiewiora.
\newblock Fast gradient-descent methods for temporal-difference learning with
  linear function approximation.
\newblock In {\em Proceedings of the 26th Annual International Conference on
  Machine Learning}, pages 993--1000, 2009.

\bibitem{tangamchit2002necessity}
Poj Tangamchit, John~M Dolan, and Pradeep~K Khosla.
\newblock The necessity of average rewards in cooperative multirobot learning.
\newblock In {\em Proceedings 2002 IEEE International Conference on Robotics
  and Automation (Cat. No. 02CH37292)}, volume~2, pages 1296--1301. IEEE, 2002.

\bibitem{thomas2016data}
Philip Thomas and Emma Brunskill.
\newblock Data-efficient off-policy policy evaluation for reinforcement
  learning.
\newblock In {\em International Conference on Machine Learning}, pages
  2139--2148. PMLR, 2016.

\bibitem{tsitsiklis1997analysis}
John~N Tsitsiklis and Benjamin Van~Roy.
\newblock Analysis of temporal-diffference learning with function
  approximation.
\newblock In {\em Advances in neural information processing systems}, pages
  1075--1081, 1997.

\bibitem{wan2021learning}
Yi~Wan, Abhishek Naik, and Richard~S Sutton.
\newblock Learning and planning in average-reward markov decision processes.
\newblock In {\em International Conference on Machine Learning}, pages
  10653--10662. PMLR, 2021.

\bibitem{wang2018finite}
Yue Wang, Wei Chen, Yuting Liu, Zhi-Ming Ma, and Tie-Yan Liu.
\newblock Finite sample analysis of the gtd policy evaluation algorithms in
  markov setting.
\newblock {\em arXiv preprint arXiv:1809.08926}, 2018.

\bibitem{xu2019two}
Tengyu Xu, Shaofeng Zou, and Yingbin Liang.
\newblock Two time-scale off-policy td learning: Non-asymptotic analysis over
  markovian samples.
\newblock {\em arXiv preprint arXiv:1909.11907}, 2019.

\bibitem{yin2021near}
Ming Yin, Yu~Bai, and Yu-Xiang Wang.
\newblock Near-optimal provable uniform convergence in offline policy
  evaluation for reinforcement learning.
\newblock In {\em International Conference on Artificial Intelligence and
  Statistics}, pages 1567--1575. PMLR, 2021.

\bibitem{zhang2020gradientdice}
Shangtong Zhang, Bo~Liu, and Shimon Whiteson.
\newblock Gradientdice: Rethinking generalized offline estimation of stationary
  values.
\newblock In {\em International Conference on Machine Learning}, pages
  11194--11203. PMLR, 2020.

\bibitem{zhang2021average}
Shangtong Zhang, Yi~Wan, Richard~S Sutton, and Shimon Whiteson.
\newblock Average-reward off-policy policy evaluation with function
  approximation.
\newblock {\em arXiv preprint arXiv:2101.02808}, 2021.

\end{thebibliography}

\begin{appendix}
\section{Tabular parameterization}
For a finite-state MDP, we assume the states are indexed as $s_1,\ldots, s_{|\mathcal{S}|}$. The feature vectors are one-hot vectors in $\mathbb{R}^{|\mathcal{S}|+1}$, i.e, 
\[\fea(s_i)=e_i \mbox{ and } \rfea=e_{|\mathcal{S}|+1}1_{\gamma=1}.\] 
Given $\theta\in \R^{|\mathcal{S}|+1}$, the value function and long-run average rewards are given by 
$V^\pi_{\theta}(s_i)=\theta_{i}$, and $\bar{r}_\theta=1_{\gamma=1}\theta_{|\mathcal{S}|+1}$ respectively. 
With this parameterization, the oracle conditional average feature is simply the transition probability under $\pi$, i.e.,
\[[\phi^\pi(s)]_i=\PP^\pi_s (s'=s_i)=\sum_{a} \pi(a|s) P(s_i|s,a),\] 
where $[x]_i$ denotes the $i$-th component of a vector $x$.

Algorithms \ref{alg:fea21}, \ref{alg:fea22} and \ref{alg:fea23} summarize the direct-SGD, TD-SGD and TD(0) updates for the tabular parameterization respectively. 
It can be seen that both TD-SGD and TD(0) feature sparse updates of the value functions.

\begin{algorithm}[htp]
 \KwIn{Data sequence $h_t=(s_t,a_t,r_t)$, sampling policy $b$, evaluation policy $\pi$, discount factor $\gamma\in(0,1]$, initial values $V_0^\pi$ (and $\rbar_0$)}
 \KwOut{Value function $\overline{V}_T(s)$ ( and average reward $\overline{R}_T$ )}
 \For{$t=1$ \KwTo $T$}{
Update $\phat_t(s_t)$ and calculate $\hat\delta_t(\theta_t, s_t)=\sum_{r} r\phat_t(r|s_t)-\rbar_t+\gamma \sum_{s'} \phat_t(s'|s_t) V_t(s')-V_t(s_t)$\\
\uIf{$\gamma=1$}{
Update $\rbar_{t+1}=\rbar_t+\eta_t \delta_t$\;}
\Else{$\rbar_{t+1}=0$}
Update $V_{t+1}(s_t)=V_t(s_t)+\eta_t \hat\delta_t(\theta_t, s_t)(1-\gamma \phat_t(s_t|s_t))$\\
	\For{$s'\neq s_t$}{
Update $V_{t+1}(s')=V_t(s')-\gamma \eta_t \hat\delta_t(\theta_t, s_t)\phat_t(s'|s_t)$\;
}
Apply $\bfP_\calC$ to $V_{t+1}(s)$'s and $\rbar_{t+1}$. 
}
Return $\overline{V}_T(s)=\frac{1}{T+1}\sum_{t=0}^T V_t(s)$ and $\overline{R}_T=\frac{1}{T+1}\sum_{t=0}^T \rbar_t$.  
\caption{Tabular policy evaluation with direct-SGD}
\label{alg:fea21}
\end{algorithm}

\begin{algorithm}[htp]
 \KwIn{Data sequence $h_t=(s_t,a_t,r_t)$, sampling policy $b$, evaluation policy $\pi$, discount factor $\gamma\in(0,1]$, initial values $V_0^\pi$ (and $\rbar_0$)}
 \KwOut{Value function $\overline{V}_T(s)$ ( and average reward $\overline{R}_T$ )}
 \For{$t=1$ \KwTo $T$}{
Update $\phat_t(s_t)$ and calculate $\hat\delta_t(\theta_t, s_t)=\sum_{r} r\phat_t(r|s_t)-\rbar_t+\gamma \sum_{s'} \phat_t(s'|s_t) V_t(s')-V_t(s_t)$\\
\uIf{$\gamma=1$}{
Update $\rbar_{t+1}=\rbar_t+\eta_t \hat{\delta}_t$\;}
\Else{$\rbar_{t+1}=0$}
Update $V_{t+1}(s_t)=V_t(s_t)+\eta_t \hat\delta_t(\theta_t, s_t)(1-\gamma \phat_t(s_t|s_t))$\\
Update $V_{t+1}(s_{t+1})=V_t(s_{t+1})-\gamma \eta_t \hat\delta_t(\theta_t, s_t)$\\
Apply $\bfP_\calC$ to $V_{t+1}(s)$'s and $\rbar_{t+1}$. 
}
Return $\overline{V}_T(s)=\frac{1}{T}\sum_{t=1}^T V_t(s)$ and $\overline{R}_T=\frac{1}{T}\sum_{t=1}^T \rbar_t$.  
\caption{Tabular policy evaluation via TD-SGD}
\label{alg:fea22}
\end{algorithm}

\begin{algorithm}[htp]
 \KwIn{Data sequence $h_t=(s_t,a_t,r_t)$, sampling policy $b$, evaluation policy $\pi$, discount factor $\gamma\in(0,1]$, initial values $V_0^\pi$ (and $\rbar_0$)}
 \KwOut{Value function $\overline{V}_T(s)$ ( and average reward $\overline{R}_T$ )}
 \For{$t=1$ \KwTo $T$}{
Update $\hat{\mu}^b(s_t)$\\
Compute $\hat{\delta}_t=-\frac{1}{\hat{\mu}^b(s_t)}\frac{\pi(a_t|s_t)}{b(a_t|s_t)}(r_t-\rbar_{t}+\gamma V_t (s_{t+1})-V_t(s))$ \\ 
\uIf{$\gamma=1$}{
Update $\rbar_{t+1}=\rbar_t-\eta_t \hat{\delta}_t$\;}
\Else{$\rbar_{t+1}=0$}
Update $V_{t+1}(s_t)=V_t(s_t)-\eta_t \hat\delta_t$\\
Apply $\bfP_\calC$ to $V_{t+1}(s)$'s and $\rbar_{t+1}$. \\
}
Return $\overline{V}_T(s)=\frac{1}{T}\sum_{t=1}^T V_t(s)$ and $\overline{R}_T=\frac{1}{T}\sum_{t=1}^T \rbar_t$.  
\caption{Tabular policy evaluation via TD(0)}
\label{alg:fea23}
\end{algorithm}


\section{Poof of the auxiliary results}
\subsection{Proof of Proposition \ref{prop:optimality}}
\begin{proof}
We note that 
\[
l(\theta)=\sum_{s}\mu^b(s)L(\theta,s).
\]
$L$ is quadratic in $\theta$ and
\begin{align*}
L(\theta,s)=\frac{1}{2}(\xi^{\pi}(s)+(-\rfea+\gamma\phi^{\pi}(s)-\fea(s))^T \theta)^2\geq 0.
\end{align*}
When $V^\pi_{\theta^*}$ solves the Bellman equation, the lower bound $0$ is achieved.

Meanwhile, if $\theta^*$ is a minimizer of $l(\theta)$, 
\[
\xi^{\pi}(s)+(-\rfea+\gamma\phi^{\pi}(s)-\fea(s))^T \theta=0,
\]
which implies that $V^\pi_{\theta^*}$ solves the Bellman equation
\end{proof}

\subsection{Poof of Lemma \ref{lem:checksutton}}
\begin{proof}
Note that 
\begin{align*}
&\sum_{s,a,s'}\mu^b(s)b(a|s)p(s'|a,s)\frac{\pi(a|s)}{b(a|s)}(\xi^\pi(s)+(-\rfea+\gamma \phi^\pi(s)-\fea(s))^T \theta) (-\rfea+\gamma \fea(s')-\fea(s))\\
=&\sum_{s}\mu^b(s)( \xi^\pi(s)+(-\rfea+\gamma \phi^\pi(s)-\fea(s))^T \theta)\left(-\rfea+\gamma \sum_{a,s'}\pi(a|s)p(s'|s,a) \fea(s')-\fea(s)\right)\\
=&\sum_{s}\mu^b(s)( \xi^\pi(s)+(-\rfea+\gamma \phi^\pi(s)-\fea(s))^T \theta)(-\rfea+\gamma  \phi^{\pi}(s)-\fea(s))\\
=&\sum_{s}\mu^b(s)\nabla L(\theta,s)=\nabla l(\theta).
\end{align*}
\end{proof}

\subsection{Poof of Lemma \ref{lem:gen}}
\begin{proof}
For SGD-based updates, the result is based on the definition of strong convexity. In particular, there exists some $\tilde \theta$ between $\theta$ and $\theta^*$, such that
\[
\langle \theta-\theta^*, \bar g(\theta)-\bar g(\theta^*)\rangle
= (\theta-\theta^*)^T\nabla^2 l(\tilde \theta)   (\theta-\theta^*).
\]
In addition, note that 
\[\nabla^2 l(\theta)=\sum_s \mu^b(s)(-\rfea+\gamma \phi^\pi(s)-\fea(s))(-\rfea+\gamma \phi^\pi(s)-\fea(s))^T,\]
which is independent of $\theta$. Thus,
\[
(\theta-\theta^*)^T\nabla^2 l(\tilde \theta)   (\theta-\theta^*)\geq c\|\theta-\theta^*\|^2.
\]

For TD(0),
\begin{align*}
\bar{g}(\theta)&=\sum_{s,a,r,s'} \mu^b(s)b(s|a)p(s'|s,a)p(r|s,a) g(\theta, z)\\
&=-\sum_s (\xi^\pi(s)+(-\zeta+\gamma \phi^\pi(s)-\fea(s))^T\theta)(\fea(s)+\zeta).
\end{align*}
It is easy to check that $\bar{g}(\theta^*)=0$ since $V_{\theta^*}$ solves the Bellman equation. 
Next,
\begin{align*}
\langle\theta-\theta^*,\bar{g}(\theta)\rangle&=\langle\theta-\theta^*,\bar{g}(\theta)-\bar{g}(\theta^*)\rangle\\
&=-\sum_s (\gamma \phi^\pi(s)-\fea(s)-\zeta)^T (\theta-\theta^*) (\fea(s)+\zeta)^T(\theta-\theta^*)\\
&=-(\theta-\theta^*)^T\left(\sum_s (\fea(s)+\zeta)(\gamma \phi^\pi(s)-\fea(s)-\zeta)^T\right) (\theta-\theta^*) \\
&=(\theta-\theta^*)^T D(\theta-\theta^*)\\ 
&=\frac{1}{2}((\theta-\theta^*)^T D(\theta-\theta^*)+(\theta-\theta^*)^T D^T(\theta-\theta^*))\\
&\geq c\|\theta-\theta^*\|^2.
\end{align*}
\end{proof} 

\subsection{Proof of Proposition \ref{lm:strongconvex}}
\begin{proof}
When $\gamma<1$, $\zeta=0$.
Under the tabular parameterization, 
\begin{align*}
\nabla^2 l=( I-\gamma P^{\pi})^T\left(\sum_{i}\mu^b(s_i)e_ie_i^T\right)( I-\gamma P^{\pi}).
\end{align*}
Since $\left(\sum_{i}\mu^b(s_i)e_ie_i^T\right)\succeq \min_s \mu^b(s)  I$ and $\|P^\pi\|\leq 1$, 
\[
\nabla^2 l\succeq \min_s \mu^b(s) ( I-\gamma  P^{\pi})^T( I-\gamma P^{\pi})
\succeq \min_s \mu^b(s)(1-\gamma )^2 I.
\]

Next, note that under the tabular parameterization, 
$D
=I-\gamma P^{\pi}$.
Then, for any vector $v$, $v^TDv=\|v\|^2-\gamma v^TP^\pi v$.
Since $\|P^{\pi}\|\leq 1$,  
\[v^TDv\geq (1-\gamma)\|v\|^2.\]
\end{proof}

\subsection{Proof of Proposition \ref{lm:strongconvex1}}
\begin{proof}
We first introduce a few notations. In  both cases, 
\[
\nabla^2 l=\sum_{s}\mu^b(s)( \phi^{\pi}(s)-\fea(s))(\phi^{\pi}(s)-\fea(s))^T+e_{|\calS|}e_{|\calS|}^T.
\] 

\textbf{Case 1.}  Recall that $\tilde{P}^\pi$  be an ${|\calS|\times |\calS|}$ matrix with the first $|\calS|-1$ columns the same as the first $|\calS|-1$ columns of $P^\pi$, while the last column is zero. Then $\phi^\pi(s_i)=(\tilde P^\pi)^Te_i$, $i=1\ldots,|\calS|$. 
\begin{align*}
\nabla^2 l(\theta)=&\sum_{s}\mu^b(s)( \phi^{\pi}(s)-\fea(s))( \phi^{\pi}(s)-\fea(s))^T+e_{|\calS|} e_{|\calS|}^T\\
=&\sum_{i=1}^{|\calS|-1}\mu^b(s_i)((\tilde P^\pi)^T e_i-e_i)(\tilde P^\pi e_i-e_i)^T+\mu^b(s_{|\calS|})(\tilde P^\pi)^T e_{|\calS|}e_{|\calS|}^T\tilde P^\pi+e_{|\calS|} e_{|\calS|}^T\\
=&(I-\tilde{P}^\pi)^T\left(\sum_{i=1}^{|\calS|-1}\mu^b(s_i)e_i e_i^T\right) (I-\tilde{P}^\pi)+\mu^b(s_{|\calS|})(\tilde P^\pi)^T e_{|\calS|}e_{|\calS|}^T\tilde P^\pi+e_{|\calS|} e_{|\calS|}^T.
\end{align*}

We next show that  $\|\tilde{P}^\pi\|<1$. 
Suppose $\|\tilde{P}^\pi\|=p$ and $v$ is the largest left eigenvector of $\tilde{P}^\pi$ i.e. $(\tilde{P}^\pi)^T v=p v$.
Since $\tilde{P}^\pi $ has all entries being non-negative, by Perron-Frobenius theorem, all entries of $v$ are non-negative, and so are $(\tilde{P}^\pi)^T v$ and $(P^\pi)^T v$.
Then, $(P^{\pi})^T v$ is larger than or equal to $(\tilde{P}^\pi)^T v$ component wise. This implies that $\|(\tilde{P}^\pi)^T v\|\leq \|(P^\pi)^T v\|$.
Since $\|P^\pi\|\leq 1$, $p\leq 1$.
Next, suppose $p=1$. Then, $v=(P^\pi)^T v$. Since $P^\pi$ is ergodic,  $v$ is a multiple of $\mu^\pi$.
Meanwhile, because $v=(\tilde{P}^\pi)^T v$, $v_{|\calS|}=0$. This contradicts that $\mu^\pi(s_{|\calS|})>0$. Thus, $p<1$.
This further implies that $v^T (\nabla^2 l) v \geq \min_s\mu^b(s)(1-p)^2\|v\|^2$.

For $D$, note that
\[
D=\sum_{i,j\neq |\calS|} \left(e_i e_i^T- P^\pi(i,j) e_i e_j^T+e_{|\calS|}e_{|\calS|}^T\right).
\]
For any vector $v$, let $\tilde{v}$ be the same as $v$ except that the last component is replaced by $0$. Then, 
\[
v^TDv= \|v\|^2- \tilde{v}^T \tilde{P}^\pi \tilde{v}.
\]
Since $\|\tilde{P}^\pi \|=p<1$, we have $v^TD v \geq (1-p)\|v\|^2$. 
%
%
%

\textbf{Case 2.} 
Let $\tilde U=(U,{\bf 0})\in \R^{|\calS|\times |\calS|}$. Note that $\fea(s_i)=\tilde U^T e_i:=\tilde u_i$. 
Thus,
\begin{align*}
\nabla^2 l=&\sum_{s}\mu^b(s)(\phi^{\pi}(s)-\fea(s))(\phi^{\pi}(s)-\fea(s))^T+e_{|\calS|}e_{|\calS|}^T\\
=&\sum_{i}\mu^b(s_i)\left(\sum_{j}P^{\pi}(i,j)\tilde u_j- \tilde u_i\right)\left(\sum_{j} P^{\pi}(i,j)\tilde u_j-\tilde u_i\right)^T+e_{|\calS|}e_{|\calS|}^T\\
=&\sum_{i}\mu^b(s_i)(\tilde{U}^T(P^{\pi})^Te_i-\tilde U^T  e_i)(\tilde U^T (P^{\pi})^Te_i-\tilde U^T e_i)^T+e_{|\calS|}e_{|\calS|}^T\\
=&\tilde{U}^T(I-P^\pi)^T\left(\sum_{i}\mu^b(s_i) e_i e_i^T \right)(I-P^\pi)\tilde{U}+e_{|\calS|}e_{|\calS|}^T.
\end{align*}
Since $P^{\pi}$ is the transition matrix and it has a spectral gap $\lambda$, the largest eigenvalue of $P^{\pi}$ is $1$, which is simple due to ergodicity.
Let $\rho$ be the eigenvalue of $P^{\pi}$ with the second largest norm, which has to be less than 1. We next show that $|\rho|\leq1-\lambda$.
To see this, first note that
$P^{\pi}{\bf 1}={\bf 1}$. Next, let $v\bot {\bf 1}$ be an eigenvector of $P^\pi$ with $P^\pi v=\rho v$.  Note that 
\[
(P^\pi)^k v = \rho^k v \mbox{ and } (\mu^\pi)^T(P^\pi)^k v=(\mu^\pi)^T v=0.
\]
Then,
\begin{align*}
2\log |\rho|&=\lim_{k\to \infty}\frac1k \log \|\rho^k v\|^2\\
&=\lim_{k\to \infty}\frac1k \log \|(P^\pi)^k v\|^2\\
&=\lim_{k\to \infty}\frac1k \log \sum_{s}|(P^\pi)^k v(s)-(\mu^\pi)^T (P^\pi)^k v|^2\\
&\leq \lim_{k\to \infty}\frac1k \log \sum_{s} \frac{\mu^\pi(s)}{\min_s \mu^\pi(s)}|(P^\pi)^k v(s)-(\mu^\pi)^T (P^\pi)^k v|^2\\
&=\lim_{k\to \infty}\frac1k\log \frac{\var_{\mu^\pi}[(P^\pi)^k v]}{\min_s \mu^\pi(s)}\\
&\leq\lim_{k\to \infty}\frac1k\log \frac{(1-\lambda)^{2k}\var_{\mu^\pi}[v]}{\min_s \mu^\pi(s)}
=2\log(1-\lambda).
\end{align*}
Based on the above observation, for any $v\bot \vec{1}$, $\|(P^\pi-I)v\|\geq \lambda\|v\|$. 
Next, for any vector $q$, we decompose it into $q=\tilde q+q_{|\calS|} e_{|\calS|}$. Set $v=\tilde{U} q=\tilde U \tilde q$, which satisfies $v\bot \vec{1}$ and $\|v\|=\|\tilde{q}\|$. Then,
\begin{align*}
&q^T\sum_{s}\mu^b(s)( \pfea(s)-\fea(s))( \pfea(s)-\fea(s))^T q+q^T e_{|\calS|} e_{|\calS|}^T q\\ 
=&v^T(P^\pi-I)^T\left(\sum_{i}\mu^b(s_i) e_i e_i^T \right)(P^\pi-I)v+q_{|\calS|}^2\\
\geq& \min_{s_i} \mu^b(s_i)\|(P^\pi-I)v\|^2+q_{|\calS|}^2\\
\geq&\min_{s_i} \mu^b(s_i) \lambda^2\|v\|^2+q_{|\calS|}^2\geq \min_{s_i} \mu^b(s_i) \lambda^2\|q\|^2. 
\end{align*}

For $D$, note that
\begin{align*}
D&=\sum_{i,j\neq |\calS|} u_i u_i^T- P^\pi(i,j)u_i u_j^T+e_{|\calS|}e_{|\calS|}^T\\
&=\sum_{i,j\neq |\calS|} \tilde{U}^Te_i e_i^T\tilde U- P^\pi(i,j)\tilde{U}^Te_i e_j^T \tilde{U} +e_{|\calS|}e_{|\calS|}^T\\
&=\sum_{i,j} \tilde{U}^Te_i e_i^T\tilde U- P^\pi(i,j)\tilde{U}^Te_i e_j^T \tilde{U} +e_{|\calS|}e_{|\calS|}^T\\
&=\tilde{U}^T \tilde U-\tilde U^T P^\pi \tilde U+e_{|\calS|}e_{|\calS|}^T
=I-\tilde U^T P^\pi \tilde U.
\end{align*}
For any vector $q=\tilde q +q_{|\calS|} e_{|\calS|}$, $v=\tilde{U}q$ satisfies $v\bot \vec{1}$ and $\|v\|=\|\tilde{q}\|$. Then,
\[
q^TDq\geq  \|q\|^2-(1-\lambda)\|\tilde{q}\|^2\geq \lambda \|q\|^2.
\]

\end{proof}

\section{Proof of Theorem \ref{thm:main}}

Recall that $l(\theta)$ is the loss function that we try to minimize. 
$g(\theta,z)$ is an oracle stochastic gradient, i.e., $\sum_{z} g(\theta, z)\mu(z):=\bar g(\theta)=\nabla l(\theta)$.
$\ghat_t(\theta,z)$ is an approximate stochastic gradient with $\|\ghat_t(\theta,z)-g(\theta,z)\|\leq e_t$.
Define $\eta_{m:n}=\sum_{t=m}^{n-1}\eta_t$ and $\eta_{m:m}=0$.

Before we prove the main results, we first present a few auxiliary lemmas.

\begin{lemma}\label{lm:mix}
Under Assumption \ref{aspt:ergodic}, take $\tau_\epsilon=\lceil |\log \epsilon|/|\log 2|\rceil \tau $. Then, 
\[
\|\Prob_{\mu}(z_{\tau_\epsilon}\in \,\cdot\,)-\mu(\cdot)\|_{TV}\leq \epsilon,
\]
and
\[
\left|\E_t \langle \bar g(\theta_t)-g(\theta_t,z_{t+\tau_{\epsilon}}),\theta_t-\theta^*\rangle\right| \leq 2C_0C_1\epsilon.
\]
\end{lemma}
\begin{proof}
For any function $f$ with $|f|\leq 1$,
we define $f_0(z):=f(z)-\sum_{z'} \mu(z') f(z')$. Note that $\sum_z \mu(z)f_0(z)=0$ and $|f_0(z)|\leq 1$.
We also define $f_1(s):=\E_{s} f_0(z_\tau)$. Note that
\[
\sum_{z}\mu(z) f_1(z)=\sum_{z} \mu(z)\Prob_z(z_\tau=z')f_0(z')=\sum_{z'} \mu(z') f_0(z')=0
\]
and
\[
|f_1(z)|= \left|\sum_{z'}(\Prob_z(z_\tau=z')-\mu(z')) f(z')\right|\leq \frac12. 
\]
We can repeat the above procedure and define a sequence as 
 $f_k(z):=\E_z[f_0(z_{k\tau})]$. Next, we show that $|f_k|\leq \frac{1}{2^k}$. This is true when $k=0$. For $k\geq 1$, suppose $|f_{k-1}(z)|\leq \frac{1}{2^{k-1}}$ for any $z\in \mathcal{Z}$. Then, we have
\[
\sum_{s}\mu(z) f_k(z)=\sum_{s} \mu(z)\Prob_z(s_{k\tau}=z') f_0(z')=\sum_{z'} \mu(z') f_0(z')=0,
\]
and
\[
|f_k(s)|=\left|\sum_{z'}(\Prob_z(z_\tau=z')-\mu(z')) f_{k-1}(z')\right| \leq \frac{1}{2}\sup_{z'}|f_{k-1}(z')|
\leq \frac{1}{2^k}.
\]
Let $n=\lceil |\log \epsilon|/|\log 2|\rceil$. Then,
\[
f_n(z)=\E_{z} f(z_{\tau_\epsilon}) -\sum_{z'} \mu(z') f(z')
\mbox{ and }
|f_n(z)|\leq\frac{1}{2^n}\leq \epsilon.
\] 
Since $f$ is any function with $|f|\leq 1$, we have $\|\Prob_{\mu}(z_{\tau_\epsilon}\in \,\cdot\,)-\mu(\cdot)\|_{TV}\leq \epsilon$.

Next 
\[\begin{split}
&\left|\E_t \langle \bar g(\theta_t)-g(\theta_t,z_{t+\tau_{\epsilon}}),\theta_t-\theta^*\rangle\right|\\
\leq&\left\|\sum_{z'}(\mu(z') - \Prob_{z_t}(z_{t+\tau_{\epsilon}}=z'))g(\theta_t,z')\right\|\|\theta_t-\theta^*\| \mbox{ by Cauchy-Schwarz inequality}\\
\leq&\epsilon C_0\|\theta_t-\theta^*\| \leq 2C_0C_1\epsilon.
\end{split}\] 
\end{proof}

\begin{lemma}
\label{lem:lip}
Under Assumptions \ref{aspt:bound}, the stochastic iterates from \eqref{eq:genSGD} satisfy
\[
\|\theta_{t+n}-\theta_t\|\leq C_0\eta_{t:t+n}.
\]
In addition, 
\[
|\langle g(\theta_t, z_{t+n})-g(\theta_{t+n},z_{t+n}), \theta_t-\theta^*\rangle| \leq 2GC_0C_1 \eta_{t:t+n}.
\]
\end{lemma}
\begin{proof}
By the convexity of $\calC$,
\[
\|\theta_{t+1}-\theta_t\|\leq \|\eta_t\ghat_t(\theta_t, z_t)\|\leq C_0\eta_t.
\]
Then,
\[
\|\theta_{t+n}-\theta_t\| \leq \sum_{k=t}^{t+n-1}\|\theta_{k+1}-\theta_k\| \leq C_0\eta_{t:t+n}.
\]
Next,
\[\begin{split}
&|\langle g(\theta_t, z_{t+n})-g(\theta_{t+n},z_{t+n}), \theta_t-\theta^*\rangle| \\
\leq& \| g(\theta_t, z_{t+n})-g(\theta_{t+n},z_{t+n})\|\|\theta_t-\theta^*\| \mbox{ by Cauchy-Schwarz inequality}\\
\leq& G\|\theta_{t+n}-\theta_t\| 2C_1
\leq 2GC_0C_1 \eta_{t:t+n}.
\end{split}\]
\end{proof}

\begin{proof}[Proof of Theorem \ref{thm:main}]
By the convexity of $\calC$ and the fact that $\theta^*\in \calC$, we have
\begin{align*}
\|\theta_{t+1}-\theta^*\|^2&\leq\|\theta_{t}-\theta^*-\eta_t \ghat_t(\theta_t, z_t) \|^2\\
&=\|\theta_{t}-\theta^*\|^2+2\eta_t \langle \ghat_t(\theta_t, z_t),\theta^*-\theta_t\rangle+\eta_t^2 \|\ghat_t(\theta_t, z_t)\|^2\\
&\leq \|\theta_{t}-\theta^*\|^2+2\eta_t \langle \ghat_t(\theta_t, z_t),\theta^*-\theta_t\rangle+C_0^2\eta_t^2.
\end{align*}
Therefore, 
\[
\langle \ghat_t(\theta_t, z_t),\theta^*-\theta_t\rangle\geq \frac{1}{2\eta_t}(\|\theta_{t+1}-\theta^*\|^2-\|\theta_{t}-\theta^*\|^2)-\frac12\eta_t C_0^2. 
\]
Next, by the convexity of $l$ and that $\bar g$ is a gradient of $l$ 
\begin{align*}
l(\theta^*)-l(\theta_t)&\geq \langle \bar g(\theta_t), \theta^*-\theta_t\rangle\\
&=\langle \ghat_t(\theta_t, z_t), \theta^*-\theta_t\rangle+\langle \bar g(\theta_t)-\ghat_t(\theta_t, z_t), \theta^*-\theta_t\rangle\\
&\geq  \frac{1}{2\eta_t}(\|\theta_{t+1}-\theta^*\|^2-\|\theta_{t}-\theta^*\|^2)-\frac12\eta_t C_0^2
+\langle \bar g(\theta_t)-\ghat_t(\theta_t, z_t), \theta^*-\theta_t\rangle.
\end{align*}
This leads to 
\begin{align*}
&\sum_{t=1}^T l(\theta_t)-l(\theta^*)\\
\leq& \sum_{t=1}^T\left(\frac{1}{2\eta_t}-\frac{1}{2\eta_{t-1}}\right)\|\theta_{t}^2-\theta^*\|^2+\frac12 C_0^2\eta_{1:(T+1)}
+\sum_{t=1}^T \langle \bar g(\theta_t)-\ghat_t(\theta_t, z_t), \theta_t-\theta^*\rangle\\
\leq& \frac{1}{\eta_T} C_1^2+\frac12 C_0^2\eta_{1:(T+1)} +\sum_{t=1}^T 
\langle \bar g(\theta_t)-\ghat_t(\theta_t, z_t), \theta_t-\theta^*\rangle. 
\end{align*}

We next establish an appropriate bound for $\langle \bar g(\theta_t)-\ghat_t(\theta_t, z_t), \theta_t-\theta^*\rangle$.
We first decompose
\begin{align*}
&\langle\bar g(\theta_t)-\ghat_t(\theta_t, z_t), \theta_t-\theta^*\rangle\\
=&\underbrace{\langle g(\theta_t,z_t)-\ghat_t(\theta_t, z_t), \theta_t-\theta^*\rangle}_{(a)}+
\underbrace{\langle g(\theta_{t+\tau_\epsilon},z_{t+\tau_\epsilon})-g(\theta_t,z_t), \theta_t-\theta^*\rangle}_{(b)}\\
&+\underbrace{\langle g(\theta_{t},z_{t+\tau_\epsilon})-g(\theta_{t+\tau_\epsilon},z_{t+\tau_\epsilon}), \theta_t-\theta^*\rangle}_{(c)}
+\underbrace{\langle \bar g(\theta_t)-g(\theta_{t},z_{t+\tau_\epsilon}), \theta_t-\theta^*\rangle}_{(d)}
\end{align*}
We bound each of term in the decomposition. For (a), we have
\[
|\langle g(\theta_t,z_t)-\ghat_t(\theta_t,z_t), \theta_t-\theta^*\rangle|
\leq\|g(\theta_t,z_t)-\ghat_t(\theta_t,z_t)\|\|\theta_t-\theta^*\|\leq 2C_1e_t.
\]
For (c), by Lemma \ref{lem:lip},
\[
|\langle g(\theta_{t},z_{t+\tau_\epsilon})-g(\theta_{t+\tau_\epsilon},z_{t+\tau_\epsilon}), \theta_t-\theta^*\rangle| \leq 2GC_0C_1 \eta_{t:t+\tau_\epsilon}
\leq 2GC_0C_1 \tau_\epsilon\eta_{t}.
\]
For (d), by Lemma \ref{lm:mix},
\[
\left|\E_t \langle \bar g(\theta_t)-g(\theta_t,z_{t+\tau_{\epsilon}}),\theta_t-\theta^*\rangle\right| \leq 2C_0C_1\epsilon.
\]
Lastly, for (b), taking the sum over $t$, we have
\begin{align*}
&\left|\sum_{t=1}^T\langle g(\theta_{t+\tau_\epsilon},z_{t+\tau_\epsilon})-g(\theta_t,z_t), \theta_t-\theta^*\rangle\right|\\
\leq &\left|\sum_{t=\tau_\epsilon+1}^{T} \langle g(\theta_{t},z_t), \theta_{t-\tau_\epsilon}-\theta_t\rangle\right|
+\left|\sum_{t=1}^{\tau_\epsilon} \langle g(\theta_{t},z_t), \theta^*-\theta_t\rangle\right|+
\left|\sum_{t=T+1}^{T+\tau_\epsilon} \langle g(\theta_{t},z_t), \theta_{t-\tau_\epsilon}-\theta^*\rangle\right|\\
\leq& C_0\sum_{t=\tau_\epsilon+1}^{T}C_0\tau_\epsilon\eta_{t-\tau_\epsilon}+2C_0C_1\tau_\epsilon+2C_0C_1\tau_\epsilon\\
\leq& C_0^2\tau_\epsilon\eta_{1:(T+1)}+4C_0C_1\tau_\epsilon.
\end{align*}
Thus,
\[\begin{split}
&\left|\E\sum_{t=1}^T \langle \bar g(\theta_t)-\ghat_t(\theta_t, z_t), \theta_t-\theta^*\rangle\right|\\
\leq& 2C_1\sum_{t=1}^T\E e_t + C_0^2\tau_\epsilon\eta_{1:(T+1)}+4C_0C_1\tau_\epsilon
+2GC_0C_1\tau_{\epsilon}\eta_{1:(T+1)}+2C_0C_1T\epsilon
\end{split}\]

Putting the bounds for (a) -- (d) together, we have
\[\begin{split}
&\E \sum_{t=1}^T \left(l(\theta_t)-l(\theta^*)\right)
=\sum_{t=1}^T \left(\E_tl(\theta_t)-l(\theta^*)\right)\\
\leq& \frac{1}{\eta_T} C_1^2+\eta_{1:(T+1)}\left(\frac12C_0^2+C_0^2\tau_\epsilon+2GC_0C_1\tau_{\epsilon}\right)
+4C_0C_1\tau_{\epsilon}+2C_0C_1T\epsilon+2C_1\sum_{t=1}^T\E e_{t}
\end{split}\]
For $\theta_T=\frac{1}{T}\sum_{t=1}^{T}\theta_t$, by the convexity of $l$, we have 
\[\begin{split}
&\E l(\bar{\theta}_T)-l(\theta^*)\leq \frac 1T \E \sum^T_{t=1}[l(\theta_t)-l(\theta^*)]\\
\leq& \frac{1}{T\eta_T }C_1^2+\frac1{T}\eta_{1:(T+1)} \left(\frac12C_0^2+C_0^2\tau_\epsilon+2GC_0C_1\tau_{\epsilon}\right)
+4C_0C_1\frac{\tau_{\epsilon}}{T}+2C_0C_1\epsilon+\frac{2C_1}{T}\sum_{t=1}^T\E e_t. 
\end{split}\]
Set $\epsilon=1/T$. Then, $\tau_{\epsilon}=O(\tau\log T)$ and
\[
\E l(\bar{\theta}_T)-l(\theta^*) = O\left(\frac{1}{T}\left(\tau\log T\sum_{t=1}^T\eta_t+\frac{1}{\eta_T}+\sum_{t=1}^T\E e_t\right)\right).
\]
\end{proof}

\section{Proof of Theorem \ref{th:main2}} \label{sec:proof_main2}
We first present some auxiliary results about our choice of the step size in \eqref{eq:step}.
\begin{lemma}
For the step size defined in \eqref{eq:step}, we have
\begin{enumerate}
\item $\exp\left(-\frac{1}{2}c\eta_{t:T}\right)\eta_t=\exp\left(-\frac{1}{2}c\eta_{k:T}\right)\eta_{k}$ for all $k,t\leq T$
\item $\eta_t=\Theta(1/t)$
\item $\exp(-c\eta_{t:T})=\Theta(t/T)$.
\end{enumerate}
\end{lemma}
\begin{proof}
For Claim 1, we note that for any $k\in \mathbb{N}$,
\[\begin{split}
\eta_T&=\exp\left(-\frac{1}{2}c\eta_{T-1}\right)\eta_{T-1}\\
&=\exp\left(-\frac{1}{2}c\sum_{t=k}^{T-1}\eta_t\right)\eta_k \mbox{ by iteration.}
\end{split}\]
For Claim 2, 
for the upper bound, we note that if $\eta_t\leq \frac{2}{ct}$, which is true when $t=1$,
\[\begin{split}
\eta_{t+1} &\leq \frac{2}{ct}\exp \left(-\frac{1}{t}\right) \mbox{ as $f(x)=x\exp(-\tfrac{1}{2}cx)$ is increasing on $[0,2/c]$}\\
&\leq \frac{2}{ct} \frac{t}{t+1} \mbox{ as $\exp(1/t)>1+1/t$}\\
&=\frac{2}{c(t+1)}.
\end{split}\]
For the lower bound, if $\eta_t\geq \frac{1}{ct}$, which is true when $t=1$,
\[\begin{split}
\eta_{t+1} &\geq \frac{1}{ct} \exp\left(-\frac{1}{2t}\right) \mbox{ as $f(x)=x\exp(-\tfrac{1}{2}cx)$ is increasing on $[0,2/c]$}\\
&\geq \frac{1}{ct} \frac{t}{t+1} \mbox{ as $\exp(\tfrac{1}{2t})\leq 1+\tfrac{1}{t}$ for $t\geq 1$}\\
&=\frac{1}{c(t+1)},
\end{split}\]


For Claim 3, since $\exp(-c\eta_{t:T})\eta_t=\eta_T$, 
$\exp(-c\eta_{t:T})=\eta_T/\eta_t$.
\end{proof}

\begin{proof}[Proof of Theorem \ref{th:main2}]
We first note that by the convexity of $\calC$,
\[\begin{split}
\|\theta_{t+1}-\theta^*\|^2 &\leq \|\theta_t - \eta_t\ghat(\theta_t,z_t)-\theta^*\|^2\\
&\leq \|\theta_t-\theta^*\|^2-2\langle \theta_t-\theta^*,\ghat(\theta_t,z_t)\rangle\eta_t+C_0^2\eta_t^2.
\end{split}\]

Next, we consider the following decomposition of $\ghat(\theta_t,z_t)$:
\begin{align*}
\ghat(\theta_t,s_t)=&\underbrace{[\ghat(\theta_t,z_t)-g(\theta_t,z_t)]}_{(a)}+\underbrace{[g(\theta_t,z_t)-g(\theta_{t+\tau_\epsilon},z_{t+\tau_\epsilon})]}_{(b)}+
\underbrace{[g(\theta_{t+\tau_\epsilon},z_{t+\tau_\epsilon})-g(\theta_{t},z_{t+\tau_\epsilon})]}_{(c)}\\
&+\underbrace{[g(\theta_{t},z_{t+\tau_\epsilon})-\bar g(\theta_t)]}_{(d)}+\underbrace{\bar g(\theta_t)}_{(e)}. 
\end{align*}
Next, we bound the inner product of each part with $\theta^*-\theta_t$ except part (b). (Part (b) will be treated separately)
For (e), since $\bar g$ is a c-contraction, 
\[
-\langle\bar{g}(\theta_t),\theta_t-\theta^*\rangle\leq -c\|\theta_t-\theta^*\|^2. 
\]
For (a),  
\[\begin{split}
-\langle \ghat(\theta_t,z_t)-g(\theta_t,z_t), \theta_t-\theta^*\rangle
&\leq \frac{1}{c}\|\ghat(\theta_t,z_t)-g(\theta_t,z_t)\|^2+\frac{1}{4}c \|\theta_t-\theta^*\|^2\\ 
&\leq\frac{1}{c}e_t^2+\frac{1}{4}c \|\theta_t-\theta^*\|^2.
\end{split}\]
For (c), by Lemma \ref{lem:lip} 
\[\|\theta_{t+\tau_\epsilon}-\theta_t\|\leq C_0\eta_{t:(t+\tau_\epsilon)}\leq C_0\tau_{\epsilon}\eta_t.\] 
Then, 
\[\begin{split}
&-\E_t\langle g(\theta_{t+\tau_\epsilon},z_{t+\tau_\epsilon})-g(\theta_{t},z_{t+\tau_\epsilon}),\theta_t-\theta^*\rangle\\
\leq& \frac1c\E_t\|g(\theta_{t+\tau_\epsilon},z_{t+\tau_\epsilon})-g(\theta_{t},z_{t+\tau_\epsilon})\|^2+\frac14c\|\theta_t-\theta^*\|^2\\
\leq& \frac1c G^2C_0^2\tau_\epsilon^2\eta_t^2 +\frac14c\|\theta_t-\theta^*\|^2,
\end{split}\]
For (d), by Lemma \ref{lm:mix},
\[\begin{split}
-\E_t\langle g(\theta_{t},z_{t+\tau_\epsilon})-\bar g(\theta_t),\theta_t-\theta^*\rangle
&\leq \frac1c\| \E_tg(\theta_{t},z_{t+\tau_\epsilon})-\nabla l(\theta_t)\|^2 +\frac14c\|\theta_t-\theta^*\|^2\\
&\leq \frac1c C_0^2\epsilon^2 +\frac14c\|\theta_t-\theta^*\|^2
\end{split}\]
Putting together the bounds for parts (a), (c), (d), (e), we have 
\[\begin{split}
\E_t\|\theta_{t+1}-\theta^*\|^2\leq& \left(1-\frac{1}{2}c\eta_t\right)\|\theta_{t}-\theta^*\|^2+\frac2c(\eta_tC_0^2 \epsilon^2
+G^2C_0^2\tau_\epsilon^2\eta_t^3+\eta_te_t^2)+C_0^2\eta_t^2\\
&-2\eta_t\E_t\langle g(\theta_t,z_t)-g(\theta_{t+\tau_\epsilon},z_{t+\tau_\epsilon}),\theta_t-\theta^*\rangle.
\end{split}\]
By iteration, we have
\begin{align*}
\E\|\theta_{T}-\theta^*\|^2\leq &\exp\left(-\frac{1}{2}c\eta_{0:T}\right)\|\theta_{0}-\theta^*\|^2\\
&+\frac2c \sum_{t=0}^{T-1}\exp\left(-\frac{1}{2}c\eta_{(t+1):T}\right)(\eta_t C_0^2\epsilon^2+\eta_t \E e_t^2+G^2C_0^2\tau_\epsilon^2\eta_t^3)\\
&+\sum_{t=0}^{T-1}\exp\left(-\frac{1}{2}c\eta_{(t+1):T}\right)C_0^2\eta_t^2\\
&\underbrace{-2\sum_{t=0}^{T-1} \exp\left(-\frac{1}{2}c\eta_{(t+1):T}\right)\eta_t \E_t\langle g(\theta_{t},z_{t})-g(\theta_{t+\tau_\epsilon},z_{t+\tau_\epsilon}),\theta_{t}-\theta^*\rangle}_{(f)}.
\end{align*}

Lastly, we develop a proper bound for (f). 
Note that the summation can be re-arranged as
\[\begin{split}
&-2\underbrace{\sum_{t=\tau_\epsilon}^{T-1}\left\langle \exp\left(-\frac{1}{2}c\eta_{(t+1):T}\right)\eta_t(\theta_t-\theta^*)- \exp\left(-\frac{1}{2}c\eta_{(t+1-\tau_\epsilon):T}\right)\eta_{t-\tau_\epsilon}(\theta_{t-\tau_\epsilon}-\theta^*),g(\theta_t,z_t)\right\rangle}_{(f1)}\\
&-2\underbrace{\sum_{t=0}^{\tau_\epsilon-1} \exp\left(-\frac{1}{2}c\eta_{(t+1):T}\right)\eta_t\langle g(\theta_{t},z_{t}),\theta_{t}-\theta^*\rangle}_{(f2)}\\
&+2\underbrace{\sum_{t=T-\tau_{\epsilon}}^{T-1}\exp\left(-\frac{1}{2}c\eta_{(t+1):T}\right)\eta_t\langle g(\theta_{t+\tau_{\epsilon}},z_{t+\tau_{\epsilon}}),\theta_{t}-\theta^*\rangle}_{(f3)}
\end{split}\]
Since
\[
\eta_t=\exp\left(-\frac{1}{2}c\eta_{(t+1-\tau_\epsilon):t}\right)\eta_{t-\tau_\epsilon},  
\]
for (f1), we have
\begin{align*}
&\left|-\left\langle \exp\left(-\frac{1}{2}c\eta_{(t+1):T}\right)\eta_t(\theta_t-\theta^*)- \exp\left(-\frac{1}{2}c\eta_{(t+1-\tau_\epsilon):T}\right)\eta_{t-\tau_\epsilon}(\theta_{t-\tau_\epsilon}-\theta^*),g(\theta_t,s_t)\right\rangle\right|\\
=&\left|-\exp\left(-\frac{1}{2}c\eta_{(t+1):T}\right)\eta_t\langle \theta_t-\theta_{t-\tau_\epsilon}, g(\theta_t,s_t)\rangle\right| \\
\leq& C_0^2\tau_\epsilon^2 \exp\left(-\frac{1}{2}c\eta_{(t+1):T}\right)\eta_t^2,
\end{align*}
where the last step follows from the fact that $\|g\|\leq C_0$ and 
\[\|\theta_t-\theta_{t-\tau_\epsilon}\|\leq  C_0\eta_{t-\tau_\epsilon:t}\leq  C_0\exp\left(\frac{1}{2}c\eta_{(t+1-\tau_\epsilon):t}\right)\eta_t \tau_\epsilon
\leq C_0\frac{t}{t-\tau_{\epsilon}+1}\eta_t \tau_\epsilon \leq  C_0\tau_{\epsilon}^2 \eta_t.\]
For (f2), since $\exp\left(-\frac{1}{2}c\eta_{t:T}\right)\eta_t=\eta_T$, we have
\[
\left|-\sum_{t=0}^{\tau_\epsilon-1} \exp\left(-\frac{1}{2}c\eta_{(t+1):T}\right)\eta_t\langle g(\theta_{t},z_{t}),\theta_{t}-\theta^*\rangle\right|
\leq 2C_0C_1\tau_{\epsilon}\eta_T
\]
Similarly, for (f3), we have
\[
\left|\sum_{t=T-\tau_{\epsilon}}^{T-1}\exp\left(-\frac{1}{2}c\eta_{(t+1):T}\right)\eta_t\langle g(\theta_{t+\tau_{\epsilon}},z_{t+\tau_{\epsilon}}),\theta_{t}-\theta^*\rangle\right|
\leq 2C_0C_1\tau_{\epsilon}\eta_T
\]
In summary,
\begin{align*}
&\E\|\theta_{T}-\theta^*\|^2\\
\leq &\exp\left(-\frac{1}{2}c\eta_{0:T}\right)\|\theta_{0}-\theta^*\|^2
+\frac2c \sum_{t=0}^{T-1}\exp\left(-\frac{1}{2}c\eta_{(t+1):T}\right)(\eta_tC_0^2 \epsilon^2+\eta_t \E e_t^2+G^2C_0^2\tau_\epsilon^2\eta_t^3)\\
&+\sum_{t=0}^{T-1}\exp\left(-\frac{1}{2}c\eta_{(t+1):T}\right)C_0^2\eta_t^2
+2\sum_{t=0}^{T-1}\exp\left(-\frac{1}{2}c\eta_{(t+1):T}\right)C_0^2\tau_\epsilon^2\eta_t^2
+8C_0C_1\tau_{\epsilon}\eta_T.
\end{align*}
Since $\exp\left(-\frac{1}{2}c\eta_{(t+1):T}\right)=\frac{t}{T}$, $\eta_t=\frac{2}{ct}$, we have
\begin{align*}
\E\|\theta_{T}-\theta^*\|^2\leq &
\frac1T\|\theta_{0}-\theta^*\|^2+\frac2c \sum_{t=1}^{T}\frac{t}{T}\left(\frac{2C_0}{ct} \epsilon^2+\frac{2}{ct} \E e_t^2+G^2C_0^2\frac{8}{c^3t^3}\tau_\epsilon^2\right)\\
&+\sum_{t=0}^{T-1}\frac{t}{T}C_0^2\frac{4}{c^2t^2}+2\sum_{t=0}^{T-1}\frac{t}{T}C_0^2\tau_{\epsilon}^2\frac{4}{c^2t^2}
+8C_0C_1\tau_{\epsilon}\frac{2}{cT}\\
=&\frac1T\|\theta_{0}-\theta^*\|^2
+\frac{1}{T}\sum_{t=1}^{T}\left(G^2C_0^2\frac{16}{c^4 t^2}\tau_{\epsilon}^2+C_0^2\frac{4}{c^2t}+C_0^2\tau_{\epsilon}^2\frac{8}{c^2t}\right)
+\frac{4C_0}{c^2}\epsilon^2\\
&+\frac{4}{c^2}\frac{1}{T}\sum_{t=1}^T \E e_t^2+\frac{1}{T}\frac{16C_0C_1\tau_{\epsilon}}{c}.
\end{align*}
Set $\epsilon=1/\sqrt{T}$. Then, $\tau_\epsilon=O(\tau\log T)$ and
\[
\E\|\theta_{T}-\theta^*\|^2= O\left(\frac{\tau^2(\log T)^3}{T}+\frac{1}{T}\sum_{t=1}^T \E e_t^2\right).
\]
Moreover, if $\tau=O(1)$ and $\E e_t^2=O(\frac{1}{t})$, we have
\[
\E\|\theta_{T}-\theta^*\|^2\leq O\left(\frac{(\log T)^3}{T}\right).
\]
\end{proof}

\section{Proof of Theorem \ref{thm:main3}} \label{sec:proof3}
Before we prove Theorem \ref{thm:main3}, we first prove an auxiliary lemma.
\begin{lemma} \label{lm:assumption}
Suppose Assumptions \ref{aspt:bound2} -- \ref{aspt:estimate2} hold, and $\pi(s|a)\leq C\mu^b(s)b(s|a)$ for some $C>0$. 
\begin{enumerate}
\item $z_t=(s_t,a_t,r_t,s_{t+1})$ is an ergodic sequence on  $\mathcal{S}\times \mathcal{A}\times \mathcal{S}$ with 
\[
\mu(s,a,s')=\mu^b(s)b(a|s)p(s'|s,a)p(r|s,a)
\]
and the  mixing time $\tau+1$. 
\item There exists a stochastic sequence $e_t$ with $\E e_t<\infty$
\[
\|\ghat (\theta_t, z_t)-g(\theta_t,z_t)\|_{\infty} \leq e_t.
\]
\item There exists $C_0>0$ such that
\[
\|\nabla l\|_{\infty},\|\ghat_t\|_{\infty},\|g\|_{\infty} \leq C_0.
\]
In addition, there exists $G>0$, such that $g(\theta,z)$ is $G$-Lipschitz continuous in $\theta$ on $\mathcal{C}$ for any $z\in \mathcal{S}\times \mathcal{A}\times \mathcal{S}$ 
\end{enumerate}
\end{lemma}
\begin{proof}
For the first claim, 
fix any function $f$ on $\mathcal{Z}$ with $|f|\leq 1$.
Let
\[
h(s)=\E [f(s_0,a,r,s_1)|s_0=s]=\sum_{a,r,s'} f(s,a,r,s') b(a|s)p(s'|s,a)p(r|s,a).
\]
Note that $|h|\leq 1$ and
\[
\mu f=\sum_{s,a,s'} \mu^b(s) b(a|s)p(s'|s,a)p(r|s,a)f(s,a,r,s')=\sum_s \mu^b(s)h(s). 
\]
Next, 
\begin{align*}
|\E[f(z_{\tau+1})|z_0]-\mu f|&=|\E[f(z_{\tau+1})|s_1]-\mu f|\\
&=|\E[\E[f(z_{\tau+1})|s_{\tau+1}]|s_1]-\mu f|\\
&=|\E[h(s_{\tau+1})|s_1]- \mu^bh| \leq \frac{1}{4}.
\end{align*}

For the second claim, recall that
\[
\hat{\delta}_t(\theta, s)=\hat{\xi}^\pi(s)+(-\zeta+\gamma \hat{\phi}^\pi(s)-\fea(s))^T\theta.
\]
Then,
\[
|\hat{\delta}_t(\theta, s)|\leq \tilde C_0 + 3\tilde C_0\tilde C_1
\]
and
\[\begin{split}
&\hat{\delta}_t(\theta, s)-\delta_t(\theta, s)\\
=&\sum_{a,r} r (\phat_t(r|s,a)-p(r|s,a))\pi(a|s)+ \gamma \sum_{a,s'} (\phat_t(s'|s,a)-p(s'|s,a))\pi(a|s)\fea(s')^T\theta\\  
\leq& \tilde C_0\tilde e_t + \tilde C_0\tilde C_1 \tilde e_t = (\tilde C_0+\tilde C_0\tilde C_1) \tilde e_t.
\end{split}\]
Under direct-SGD,
\begin{align*}
\ghat_t(\theta,z)-g(\theta,z)=&\gamma \hat{\delta}_t(\theta, s) \sum_{s,a,s'} (p(s'|a,s)-\phat_t (s'|a,s))\pi(a|s)\fea(s')\\
&+ (\hat{\delta}_t(\theta, s)-\delta_t(\theta, s))(-\zeta+\gamma \phi^\pi(s)-\fea(s))\\
\leq& (\tilde C_0 + 3\tilde C_0\tilde C_1)\tilde C_0\tilde e_t+3\tilde C_0(\tilde C_0+\tilde C_0\tilde C_1)\tilde e_t
=\tilde C_0^2(4+6\tilde C_1)\tilde e_t.
\end{align*}
Under TD-SGD,
\begin{align*}
\ghat_t(\theta,z)-g(\theta,z)=&(\hat{\delta}_t(\theta, s)-\delta_t(\theta, s))\frac{\pi(a|s)}{b(a|s)}(-\zeta+\gamma \fea(s')-\fea(s))\\
\leq&3C\tilde C_0(\tilde C_0+\tilde C_0\tilde C_1)\tilde e_t.
\end{align*}
Under TD(0),
\begin{align*}
\ghat_t(\theta,z)-g(\theta,z)=&\left(\frac{1}{\mu^b(s)}-\frac{1}{\hat\mu^b(s)}\right)\frac{\pi(a|s)}{b(a|s)}(r+(-\zeta+\gamma\fea(s')-\fea(s))^T\theta)(\fea(s)+\zeta)\\
\leq&C(\tilde C_0+3\tilde C_0\tilde C_1)2\tilde C_0\tilde e_t.
\end{align*}

For the third claim, we first note that
\[
|\delta(\theta, s)|=|\xi^{\pi}(s)+(-\zeta+\gamma\phi^{\pi}(s)-\fea(s))^T\theta| \leq \tilde C_0+3\tilde C_0 \tilde C_1.
\]
Under direct-SGD,
\[
\|g(\theta,z)\|=|\delta(\theta, s)|\|-\zeta+\gamma \phi^\pi(s)-\fea(s)\| \leq (\tilde C_0+2\tilde C_0\tilde C_1)3\tilde C_0
\]
Similarly, 
$\|\ghat_t(\theta,z)\|\leq (M_0+2M_0C_1)3M_0$
and 
\[
\|\nabla l(\theta)\|\leq \left\|\sum_{s,a,s'}\mu^b(s)b(a|s)p(s'|s,a)g(\theta,(s,a,r,s'))\right\| \leq (\tilde C_0+2\tilde C_0 \tilde C_1)3\tilde C_0.
\]
Next, note that
\[\begin{split}
\|g(\theta_1,z)-g(\theta_2,z)\|
&=\|(\delta(\theta_1, s)-\delta(\theta_2, s))(-\zeta+\gamma \phi^\pi(s)-\fea(s))\|\\
&=|(-\zeta+\gamma\phi^{\pi}(s)-\fea(s))(\theta_1-\theta_2))|\cdot\|-\zeta+\gamma \phi^\pi(s)-\fea(s)\|\\
&\leq 9M_0^2\|\theta_1-\theta_2\|.
\end{split}\]
Under TD-SGD,
\[
\|g(\theta,z)\|=|\delta(\theta, s)|\frac{\pi(s|a)}{b(s|a)}\|-\zeta+\gamma \fea(s')-\fea(s)\| \leq (\tilde C_0+2\tilde C_0\tilde C_1)3C\tilde C_0
\]
Similarly, $\|\ghat_t(\theta,z)\|\leq (\tilde C_0+2\tilde C_0\tilde C_1)3C\tilde C_0$ and $\|\bar g(\theta)\| \leq (\tilde C_0+2\tilde C_0\tilde C_1)3C\tilde C_0$.\\
Next, note that
\[\begin{split}
\|g(\theta_1,z)-g(\theta_2,z)\|
&=\left\|(\delta(\theta_1, s)-\delta(\theta_2, s))\frac{\pi(a|s)}{b(a|s)}(-\zeta+\gamma \fea(s')-\fea(s))\right\|\\
&=\frac{\pi(a|s)}{b(a|s)}|(-\zeta+\gamma\fea(s')-\fea(s))(\theta_1-\theta_2))|\cdot\|-\zeta+\gamma \fea(s')-\fea(s)\|\\
&\leq 9C\tilde C_0^2\|\theta_1-\theta_2\|.
\end{split}\]
Under TD(0),
\[\begin{split}
\|g(\theta,z)\|&=\frac{\pi(s|a)}{\mu^b(s)b(s|a)}|(r+(-\zeta+\gamma\fea(s')-\fea(s))^T\theta)|\|\fea(s)+\zeta\|\\
&\leq C(\tilde C_0+3\tilde C_0\tilde C_1)2\tilde C_0
\end{split}\]
Similarly, $\|\ghat_t(\theta,z)\|\leq C(\tilde C_0+3\tilde C_0\tilde C_1)2\tilde C_0$ and $\|\bar g(\theta)\| \leq C(\tilde C_0+3\tilde C_0\tilde C_1)2\tilde C_0$.\\
Next, note that
\[\begin{split}
\|g(\theta_1,z)-g(\theta_2,z)\|
&=\frac{\pi(s|a)}{\mu^b(s)b(s|a)}|(-\zeta+\gamma \fea(s')-\fea(s))^T(\theta_1-\theta_2)|\|\fea(s)+\zeta\|\\
&\leq 6C\tilde C_0^2\|\theta_1-\theta_2\|.
\end{split}\]
\end{proof}

\begin{proof}[Proof of theorem \ref{thm:main3}]
With the linear parameterization of the value function, the loss function $l(\theta)$ defined in \eqref{eq:loss} is convex. 
In addition, from Lemma \ref{lm:assumption}, 
under Assumptions \ref{aspt:bound2} -- \ref{aspt:estimate2}, Assumptions \ref{aspt:ergodic} -- \ref{aspt:bound}
hold.
Thus, Theorem \ref{thm:main} holds with $\E e_t=O(1/\sqrt{t})$, i.e., 
\[\E l(\bar{\theta}_T)-l(\theta^*)=O\left(\log T/\sqrt{T}\right).\]
If $l(\theta)$ is strongly convex, $\bar g$ for the SGD-based updates is a c-contraction by Lemma \ref{lem:gen}.
Then, by Lemma \ref{lm:assumption}, Theorem \ref{th:main2} holds with $\E e_t^2=O(1/t)$, i.e.,
\[
\E\|\theta_T-\theta^*\|^2=O((\log T)^3/T).
\]
By strong convexity of $l$ and boundlessness of $\bar{g}$, we further have $\E l(\bar{\theta}_T)-l(\theta^*)=O\left((\log T)^3/T\right)$.
Similarly, for the TD(0) update, if $\bar g$ is a c-contraction, Theorem \ref{th:main2} holds with $\E[e_t^2]=O(1/t)$.
\end{proof}

\section{Proof of Theorem \ref{thm:main1}}
\begin{proof}
Note that because $l(\theta^*)=0$, 
\[
l(\theta)-l(\theta^*)=\sum \mu^b (s)\frac12 \left((-\rfea+\gamma \phi^{\pi}(s)-\fea(s))^T(\theta-\theta^*)\right)^2.
\]
Let $f(s)=(\theta-\theta^*)^T\fea(s)=V_\theta(s)-V^{\pi}(s)$ and $\chi=(\theta-\theta^*)^T \rfea=\bar r_{\theta}-\bar r^{\pi}$. 
Note that because
\[
P^\pi f(s)= \sum_{s^{\prime}} P^\pi(s,s^{\prime}) (\theta-\theta^*)^T\fea(s')=(\theta-\theta^*)^T\phi^{\pi}(s),
\]
\[
l(\theta)-l(\theta^*)=\frac12 \sum \mu^b (s) (-\chi+\gamma P^\pi f(s)-f(s))^2\geq \frac{1}{2C}
\sum \mu^\pi (s) (-\chi+\gamma P^\pi f(s)-f(s))^2.
\]

Let $\bar{f}=\sum \mu^\pi(s) f(s)$. Then,
\begin{align*}
&\sum \mu^\pi (s) (\gamma P^\pi f(s)-f(s)-\chi)^2\\
=&\sum \mu^\pi (s) (\gamma (P^\pi f(s)-\bar{f})-(f(s)-\bar{f})-((1-\gamma)\bar{f}+\chi))^2\\
=&\var_{\mu^\pi}{f}+\gamma^2\var_{\mu^\pi}(P^\pi f)-2\gamma\text{cov}_{\mu^\pi}(f, P^\pi f)+((1-\gamma)\bar{f}+\chi)^2\\
\geq& \var_{\mu^\pi}{f}+\gamma^2\var_{\mu^\pi}(P^\pi f)-2\gamma \sqrt{\var_{\mu^\pi}(f)\var_{\mu^\pi}(P^\pi f)}+((1-\gamma)\bar{f}+\chi)^2
\end{align*}
If we let $b=\sqrt{\var_{\mu^\pi}(f)}$,
since the function $h(x)=\gamma^2 x^2-2\gamma x b$ is decreasing on $[0, b/\gamma]$ and 
\[\var_{\mu^\pi}(P^\pi f)\leq (1-\lambda)^2\var_{\mu^\pi}(f)\] 
(because $\lambda$ is the spectral gap of $P^{\pi}$), 
\[
l(\theta)-l(\theta^*)\geq \frac{1}{2C}\left( \var_{\mu^\pi}(f)(1-\gamma(1-\lambda))^2+((1-\gamma)\bar{f}+\chi)^2\right).
\]
When $\gamma<1$, $\chi=0$ and
\[
l(\theta)-l(\theta^*)\geq \frac{1}{2C}\left( \var_{\mu^\pi}(f)(1-\gamma(1-\lambda))^2+(1-\gamma)^2\bar{f}^2\right).
\]
Then,
\[
\sum  \mu^\pi(s) (V^*(s)-V_{\theta}(s))^2=\var_{\mu^\pi} (f)+\bar{f}^2
\leq \frac{2C}{(1-\gamma(1-\lambda))^2}(l(\theta)-l(\theta^*)). 
\]
When $\gamma=1$, 
\[
l(\theta)-l(\theta^*)\geq \frac{1}{2C} \left( \var_{\mu^\pi}(f)\lambda^2+\chi^2\right).
\]
Then, 
\[\begin{split}
&\min_{a}\sum  \mu^\pi(s) (V^{\pi}(s)+a-V_{\theta}(s))^2+(\bar r^{\pi}-\bar r_\theta)^2\\
\leq& \sum  \mu^\pi(s) (V^{\pi}(s)- \mu^\pi V^{\pi}-V_{\theta}(s)+ \mu^\pi V_{\theta})^2+(\bar r^{\pi}-\bar r_\theta)^2\\
=&\var_{\mu^\pi}(f)+\eta^2
\leq \frac{2C}{\lambda^2}(l(\theta)-l(\theta^*)).
\end{split}\]
\end{proof}

\section{Proof of Theorem \ref{th:approx_iteration}}
\begin{proof}
To simplify the notation, we write $\pi=\pi_k$ and $\sigma=\pi_{k+1}$. 
Let 
\[D^\pi(s)=|V^\pi(s)-\Vhat^{\pi}(s)|.\]
We also define some notations for the transition probabilities.
Let $\tilde r(s,a)=\sum_r r p(r|s,a)$ and $\E^\pi r(s)=\sum_a \pi(a|s)r(s,a)$.
Let $P^{\sigma}_{k}(s,s')$ denote the probability of transitioning from $s$ to $s'$ after applying policy $\sigma$ for $k$ time units.
Let $P^{\sigma,\pi}_{k-1,1}(s,s')$ denote the probability of transitioning from $s$ to $s'$ after applying $\sigma$ for the first $k-1$ units of time and $\pi$ for the last  epoch.
Denote $V^{\sigma,\pi}_{m}$ as value function if we apply $\sigma$ for the first $m$ units of time and then $\pi$ thereafter, i.e., 
\[
V^{\sigma,\pi}_{m}(s)=\sum_{j=0}^k \gamma^j \sum_{s_j,a_j}P^{\sigma}_j(s,s_j)\pi(a_j|s_j) \bar r(s_j,a_j) +\gamma^m \sum_{s_m} V^\pi(s_m) P^{\sigma,\pi}_{m,1}(s,s_m). 
\]

Note that
\[
V^{\sigma,\pi}_{m}(s)=\sum_{a}\sigma(a|s)  \bar r(s,a)+\gamma\sum_{s'}V^{\sigma,\pi}_{m-1}(s')P_1^{\sigma}(s,s').
\]
We will next show that 
\begin{equation}\label{eq:claim}
V^{\sigma,\pi}_{m-1}(s)\leq V^{\sigma,\pi}_{m}(s)+\gamma^m \sum_{s'}D^\pi(s')(P^{\sigma,\pi}_{m-1,1}(s,s')+P^{\sigma}_{m}(s,s'))
\end{equation}
First,
\begin{align*}
V^{\pi}(s)&=\E^{\pi}r(s) +\gamma \sum_{s'} V^{\pi}(s') P^{\pi}(s,s')\\
&\leq \E^{\pi}r(s) +\gamma \sum_{s'}\Vhat^\pi(s') P^{\pi}(s,s')+\gamma \sum_{s'}D^\pi(s')P^\pi(s,s')\\
&\leq \E^{\sigma}r(s) +\gamma \sum_{s'} \Vhat^\pi(s') P^{\sigma}(s,s')+\gamma \sum_{s'}D^\pi(s')P^\pi(s,s')\\
& \text{by the optimality of $\sigma$ under $\Vhat^\pi$}\\
&\leq \E^{\sigma}r(s) +\gamma \sum_{s'} V^\pi(s') P^{\sigma}(s,s')+\gamma \sum_{s'}D^\pi(s')(P^\pi(s,s')+P^{\sigma}(s,s'))\\
&= V^{\sigma,\pi}_{1}(s)+\gamma \sum_{s'}D^\pi(s')(P^\pi(s,s')+P^{\sigma}(s,s')).
\end{align*}
Second, suppose the claim \eqref{eq:claim} holds for $m-1$.
\begin{align*}
V^{\sigma,\pi}_{m}(s)&=\E^{\sigma}r(s) +\gamma \sum_{s'} V_{m-1}^{\pi,\sigma}(s') P^{\sigma}(s,s') \\
&\leq \E^{\sigma}r(s) +\gamma \sum_{s'} \left[V_{m}^{\pi,\sigma}(s')+\gamma^{m}\sum_{s''}D^\pi(s'')(P^{\sigma,\pi}_{m-1,1}(s',s'')+P^{\sigma}_{m}(s',s''))\right]P^{\sigma}(s,s') \\
&=V^{\sigma,\pi}_{m+1}(s) +\gamma^{m+1} \sum_{s''}(P^{\sigma,\pi}_{m,1}(s,s'')+P^{\sigma}_{m+1}(s,s''))D^\pi(s'').
\end{align*}
We have thus proved \eqref{eq:claim}.

Since $V^{\sigma}(s)=\lim_{m\to\infty} V^{\sigma,\pi}_{m}(s)$, $V^{\sigma,\pi}_{1}(s)\leq V^{\sigma}(s)+\Delta^\pi(s)$, where
\[
\Delta^\pi(s)=\sum_{m=1}^\infty \gamma^{m} \sum_{s'}(P^{\sigma,\pi}_{m,1}(s,s')+P^{\sigma}_{m+1}(s,s'))D^\pi(s').
\]
Then,
\begin{align*}
&V^*(s)-V^{\sigma}(s)\leq V^{\pi^*,\pi^*}_{1}-V^{\sigma,\pi}_{1}(s)+\Delta^{\pi}(s) \\
=&\E^{\pi^*}r(s) + \gamma \sum_{s'} V^{\pi^*}(s') P^{\pi^*}(s,s')
-\left(\E^{\sigma}r(s) + \gamma \sum_{s'} V^\pi(s') P^{\sigma}(s,s')\right)+\Delta^{\pi}(s)\\
\leq& \E^{\pi^*}r(s) +\gamma \sum_{s'} V^{\pi^*}(s') P^{\pi^*}(s,s')-\E^{\sigma}r(s) 
-\gamma \sum_{s'} \Vhat^\pi(s') P^{\sigma}(s,s') \\
&+\gamma \sum_{s'}D^\pi(s') P^{\pi}(s,s')+\Delta^{\pi}(s)\\
\leq& \E^{\pi^*}r(s) + \gamma  \sum_{s'} V^{\pi^*}(s') P^{\pi^*}(s,s')
-\E^{\pi^*}r(s) - \gamma  \sum_{s'} \Vhat^\pi(s') P^{\pi^*}(s,s')\\
&+\gamma \sum_{s'}D^\pi(s') P^{\pi}(s,s')+\Delta^{\pi}(s) \quad \text{By the optimality of $\sigma$ under $\Vhat^\pi$}\\
\leq& \gamma \sum_{s'} (V^{\pi^*}(s')-V^{\pi}(s')) P^{\pi^*}(s,s')+\gamma \sum_{s'}D^\pi(s') P^{\pi^*}(s,s')+\gamma \sum_{s'}D^\pi(s') P^{\pi}(s,s')+\Delta^{\pi}(s).
\end{align*}
For $\gamma<1$, because $\mu^{\pi^*}(s')=\sum_s \mu^{\pi^*}(s)P^{\pi^*}(s,s')$,
\[\begin{split}
\sum_{s} \mu^{\pi^*}(s)[V^*(s)-V^{\sigma}(s)]\leq& \gamma \sum_{s} \mu^{\pi^*}(s)[V^{\pi^*}(s)-V^{\pi}(s)]
+\gamma \sum_{s}\mu^{\pi^*}(s)D^\pi(s)\\
&+\gamma\sum_{s}\sum_{s'}\mu^{\pi^*}(s)P^{\pi}(s,s')D^\pi(s')
+ \sum_s\mu^{\pi^*}(s)\Delta^{\pi}(s).
\end{split}\]
Lastly, since $\mu^{\pi^*}\leq C\mu^{\sigma}\leq C^2\mu^{\pi}$, we have
\[
\E\sum_{s}\mu^{\pi^*}(s)D^\pi(s)\leq C^2\E\sum_{s}\mu^{\pi}(s)D^\pi(s)\leq C^2\epsilon_k
\] 
and
\[\begin{split}
\E\sum_{s'}\sum_{s}\mu^{\pi^*}(s)P^{\pi}(s,s')D^\pi(s')
&\leq C^2\E\sum_{s'}\sum_{s}\mu^{\pi}(s)P^{\pi}(s,s')D^\pi(s')\\
&= C^2\E\sum_{s'}\mu^{\pi}(s')D^\pi(s') \leq C^2\epsilon_k.
\end{split}\]
Moreover,
\[\begin{split}
\E\sum_s\mu^{\pi^*}(s)\Delta^{\pi}(s)=&\E\sum_{s}\mu^{\pi^*}(s) \sum_{m=1}^\infty \gamma^m\sum_{s'}P^{\sigma,\pi}_{m,1}(s,s')D^\pi(s')\\
&+\E\sum_{s}\mu^{\pi^*}(s) \sum_{m=1}^\infty \gamma^m\sum_{s'}P^{\sigma}_{m+1}(s,s')D^\pi(s'),
\end{split}\]
where
\begin{align*}
\E\sum_{s}\mu^{\pi^*}(s)\sum_{s'}P^{\sigma,\pi}_{m,1}(s,s')D^\pi(s') &\leq C\E\sum_{s'}\sum_{s}\mu^{\sigma}(s) P^{\sigma,\pi}_{m,1}(s,s')D^\pi(s')\\
&\leq C\E\sum_{s'}\sum_{s}\mu^{\sigma}(s) P^{\pi}(s,s')D^\pi(s')\\
&\leq C^2\E\sum_{s'}\sum_{s}\mu^{\pi}(s) P^{\pi}(s,s')D^\pi(s')\\
&=C^2\E\sum_{s'}\mu^{\pi}(s')D^\pi(s')\leq C^2 \epsilon_k
\end{align*}
and
\begin{align*}
\E\sum_{s}\mu^{\pi^*}(s)\sum_{s'}P^{\sigma}_{m+1}(s,s')D^\pi(s') &\leq C\E\sum_{s'}\sum_{s}\mu^{\sigma}(s) P^{\sigma}_{m+1}(s,s')D^\pi(s')\\
&=C\E\sum_{s'}\mu^{\pi}(s')D^\pi(s')\leq C^2 \epsilon_k. 
\end{align*}
Then,
\[
\E\sum_{s} \mu^{\pi^*}(s)[V^*(s)-V^{\sigma}(s)]\leq \gamma \E\sum_{s}\mu^{\pi^*}(s) [V^{\pi^*}(s)-V^{\pi}(s)]
+\left(2\gamma+\frac{2\gamma}{1-\gamma}\right)C^2\epsilon_k.
\]
The final claim can be obtained through Gronwall's inequality. 
\end{proof}

\section{Proof of Corollary \ref{cor:complex}}
\begin{proof}
Let $\epsilon$ denote the target accuracy level. In particular, we would like to find an appropriate sample size $T$ (number of iterations in policy evaluation) and number of policy updates $K$ (number of iterations in policy iteration), such that
\begin{equation} \label{eq:accuracy}
\E \sum_{s} \mu^*(s)(V^{*}(s)-V^{\pi_{K}}(s))=\epsilon.
\end{equation}

From Theorem \ref{th:approx_iteration}, we have
\[
\E \sum_{s} \mu^*(s)(V^{*}(s)-V^{\pi_{K}}(s))\leq \gamma^K\E\sum_s\mu^*(s)(V^{\pi^*}(s)-V^{\pi_{0}}(s))+\frac{\sum_{k=1}^K\gamma^{K-k}C^2\epsilon_k}{1-\gamma}
\]
where
\[
\epsilon_k = \E \sum_s \mu^{\pi_k}(s)|V^{\pi_k}(s)-\Vhat^{\pi_k}(s)|.
\]
Then, to achieve an $\epsilon$ accuracy as defined in \eqref{eq:accuracy}, 
we need $\gamma^K=C\epsilon$ and $\epsilon_k=C\epsilon(1-\gamma)^2$.
This implies that we require 
\[
K=O\left(\frac{\log\epsilon}{\log\gamma}\right)=O\left(\frac{\log(1/\epsilon)}{1-\gamma}\right).
\]
In addition, under Assumptions \ref{aspt:bound2} -- \ref{aspt:ergodic2}, Theorems \ref{thm:main} and \ref{thm:main1} indicate that
$\epsilon_k^2=O(\log T/\sqrt{T})$. Thus, we require
\[
T=O \left(\frac{\log(1/\epsilon)\log(1/(1-\gamma))}{(1-\gamma)^8\epsilon^4}\right).
\]
In this case, the total complexity is
\[
TK = O \left(\frac{\log(1/\epsilon)^2\log(1/(1-\gamma))}{(1-\gamma)^9\epsilon^4}\right).
\]
If $\bar g$ is a c-contraction, Theorems \ref{th:main2} indicates that
$\epsilon_k^2=O((\log T)^2/T)$. Then,
\[
T=O\left(\frac{\log(1/\epsilon)\log(1/(1-\gamma))}{(1-\gamma)^4\epsilon^2}\right).
\]
In this case, we achieve an improved complexity of
\[
TK = O\left(\frac{\log(1/\epsilon)^2\log(1/(1-\gamma))}{(1-\gamma)^5\epsilon^2}\right).
\]
\end{proof}
\end{appendix}
\end{document}